\documentclass{article}
\usepackage[letterpaper,top=2cm,bottom=2cm,left=3cm,right=3cm,marginparwidth=1.75cm]{geometry}

\usepackage{calc}

\usepackage{amsmath}
\usepackage{amssymb}
\usepackage{amsthm}
\usepackage[numbers]{natbib}   
\usepackage{bm}
\usepackage{graphicx}
\usepackage[colorlinks=true, allcolors=blue]{hyperref}
\usepackage{tikz}
\usepackage{xcolor}        
\usepackage{tikz}          
\usetikzlibrary{angles,quotes,calc}  
\usetikzlibrary{arrows}
\usepackage{subcaption}
\usepackage{mathrsfs}
\usepackage{tikz}
\usepackage{pgfplots}
\usepgfplotslibrary{fillbetween}
\usepackage{wrapfig}
\theoremstyle{plain}
\newtheorem{fact}{Fact}
\newtheorem{theorem}{Theorem}
\newtheorem{lemma}[theorem]{Lemma}
\newtheorem{corollary}[theorem]{Corollary}
\newtheorem{proposition}[theorem]{Proposition}

\theoremstyle{definition}
\newtheorem{definition}[theorem]{Definition}

\newtheorem{experiment}[theorem]{Experiment}

\newcommand{\Rdk}{\mathbb{R}^{d\times k}}

\newcommand{\hmu}{\bm{\mu}^\natural}
\newcommand{\bUnk}{\mathcal{U}_{n,k}}
\newcommand{\bUkk}{\mathcal{U}_{k,k}}
\newcommand{\la}{\big\langle}
\newcommand{\ra}{\big\rangle}
\newcommand{\bmF}{\bm{F}}
\newcommand{\bmu}{\bm{\mu}}
\newcommand{\bmpi}{\bm{\pi}}

\newcommand{\trace}{\textnormal{trace}}
\newcommand{\diag}{\textnormal{diag}}

\newcommand{\cA}{\mathcal{A}}

\usepackage{algorithm}      
\usepackage{algorithmic}

\usepackage{float}          
\usepackage{booktabs}   
\usepackage{siunitx}    
\usepackage{changes}

\title{BalLOT: Balanced $k$-means clustering with optimal transport}

\author{
Wenyan Luo\thanks{Department of Mathematics, The Ohio State University, Columbus, Ohio, USA}
\and
Dustin G.\ Mixon\footnotemark[1] \thanks{Translational Data Analytics Institute, The Ohio State University, Columbus, Ohio, USA}
}

\date{}

\begin{document}

\maketitle
\begin{abstract}
We consider the fundamental problem of balanced $k$-means clustering.
In particular, we introduce an optimal transport approach to alternating minimization called BalLOT, and we show that it delivers a fast and effective solution to this problem.
We establish this with several theoretical guarantees and a variety of numerical experiments.
On the theory front, we first prove that for generic data, BalLOT produces integral couplings at each step.
Next, we perform a landscape analysis to provide theoretical guarantees for both exact and partial recoveries of planted clusters under the stochastic ball model.
We also propose initialization schemes that achieve one-step recovery of planted clusters.
To conclude, we present numerical experiments that corroborate our theoretical results.
\end{abstract}

\section{Introduction}

Clustering is a fundamental task in machine learning that aims to partition data points into disjoint clusters.
One of the most famous clustering problems is \textit{$k$-means clustering}, which seeks a partition that minimizes the total variance within clusters.
Minimizing the $k$-means objective frequently results in clusters of different sizes, which is not acceptable for some use cases, where a \textit{balanced clustering} is required.
Such use cases might involve wireless sensor networks \cite{shu2005power}, frequency-sensitive competitive learning \cite{banerjee2004frequency}, or market basket analysis \cite{ghosh2005clustering}.
This suggests the \textbf{balanced $k$-means problem}, in which each cluster is constrained to have the same size.
Specifically, when clustering $n$ data points $\{\bm{x}_i\}_{i\in[n]}\subseteq\mathbb{R}^d$ into $k$ clusters (for some $k$ that divides $n$), we seek to solve
\begin{equation}
\label{problem: balanced kmeans}
\begin{aligned}
     & \text{minimize} && \quad \sum_{j\in [k]}\sum_{i\in C_j} \bigg\|\bm{x}_i - \frac{1}{|C_j|} \sum_{l \in C_j}\bm{x}_l \bigg\|^2 \\[5pt]
   & \text{subject to} &&  \quad C_1\sqcup \cdots\sqcup C_k = [n], \qquad |C_j| = \frac{n}{k} \quad \forall\, j\in [k].
\end{aligned}
\end{equation}
(Here and throughout, we denote $[n]:=\{1,\ldots,n\}$.)
Like many other clustering problems, balanced $k$-means is generally difficult to solve.
In fact, it is known to be \textsf{NP}-hard even when $n/k=3$ \cite{pyatkin2017nphardness}.

\subsection{Existing approaches}

Two common approaches to balanced $k$-means are semidefinite programming and alternating minimization.
For the semidefinite programming approach, Amini and Levina \cite{amini2018semidefinite} found a modification of the Peng--Wei relaxation of $k$-means \cite{peng2007approximating} that gives a semidefinite relaxation of balanced $k$-means.
Let $\bm{D}=[D_{ii'}]\in\mathbb{R}^{n\times n}$ denote the matrix of squared distances, i.e., $D_{ii'}=\|\bm{x}_i-\bm{x}_{i'}\|^2$.
The \textit{Amini--Levina SDP} is then given by
\begin{equation}
\label{prob: Amini Levina SDP}
\begin{aligned}
& \text{minimize} &&\quad \trace(\bm{D}\bm{Z}) \\
& \text{subject to} && \quad \diag(\bm{Z}) = \frac{k}{n}\bm{1}_n, \quad \bm{Z}\bm{1}_n  = \bm{1}_n, \quad \bm{Z} \succeq 0, \quad \bm{Z}\ge 0.
\end{aligned}
\end{equation}
(Here and throughout, $\bm{1}_n$ denotes the all-ones vector in $\mathbb{R}^n$.)
To see the connection with balanced $k$-means, every balanced partition $C_1,\ldots,C_k$ determines the cluster matrix
\[
\bm{Z}_C:=\sum_{j\in[k]}\frac{1}{|C_j|}\bm{1}_{C_j}\bm{1}_{C_j}^T,
\]
where $\bm{1}_{C_j}\in\mathbb{R}^n$ is the indicator vector of $C_j$.
This matrix is feasible for \eqref{prob: Amini Levina SDP}, and
\[
\trace(\bm{D}\bm{Z}_C)
=2\sum_{j\in[k]}\sum_{i\in C_j}\bigg\|\bm{x}_i-\frac{1}{|C_j|}\sum_{l\in C_j}\bm{x}_l\bigg\|^2.
\]
Thus, modulo a factor of $2$, we may view the Amini--Levina SDP as a relaxation of the balanced $k$-means problem.
The Peng--Wei SDP relaxes further by imposing the weaker trace constraint $\trace(\bm{Z})=k$ in place of the uniform diagonal constraint in \eqref{prob: Amini Levina SDP}, thereby permitting clusters of unequal size.
Current theory for the Amini--Levina SDP includes a proximity condition for exact recovery \cite{li2020when}, as well as partial and exact recovery guarantees under sub-Gaussian mixture models \cite{fei2022hidden}.
Despite these guarantees, the decision variable is an $n\times n$ matrix, and semidefinite programming is known to exhibit prohibitively long runtimes when $n$ is large \cite{majumdar2020recent}.

For a faster approach, one might consider alternating minimization.
Given a partition $C_1,\ldots,C_k$, denote the corresponding centroids by
\[
\bmu_j:=\frac{1}{|C_j|}\sum_{i\in C_j}\bm{x}_i,
\qquad j\in[k].
\]
Lloyd's algorithm alternates between assigning every data point to its closest current centroid and replacing each centroid by the average of its assigned points.
This method is fast, and it has been analyzed under Gaussian and sub-Gaussian mixture models \cite{awasthi2018clustering,lu2016statistical}, but the resulting clusters need not be balanced.
A balanced alternative makes $n/k$ copies of each current centroid and computes a minimum-cost matching between the $n$ replicated centroids and the $n$ data points \cite{malinen2014balanced}.
Because every centroid has exactly $n/k$ copies and every data point is matched once, the resulting partition is balanced.
This assignment step can be implemented with the Hungarian algorithm, but its $O(n^3)$ runtime makes the per-iteration cost large when $n$ is large.

\subsection{An optimal transport approach}

It turns out that optimal transport allows one to enjoy the computational advantages of alternating minimization while simultaneously respecting the requirement of balanced cluster sizes.
The key idea is to formulate the balanced cluster-assignment step as a rectangular transportation problem.

First, we introduce variables
\[
\bmF=[F_{ij}]\in\mathbb{R}^{n\times k}
\qquad
\text{and}
\qquad
\bmu=[\bmu_1\ \cdots\ \bmu_k]\in\mathbb{R}^{d\times k}.
\]
Here, $\bmF$ represents an assignment of the $n$ data points to $k$ clusters, while the columns of $\bmu$ denote corresponding cluster centroids.
For a balanced partition, we set $F_{ij}=1/n$ if $i\in C_j$ and $F_{ij}=0$ otherwise.
For a fixed partition, the identity
\[
\sum_{i\in C_j}\|\bm{x}_i-\bmu_j\|^2
=
\sum_{i\in C_j}\bigg\|\bm{x}_i-\frac{1}{|C_j|}\sum_{l\in C_j}\bm{x}_l\bigg\|^2
+
|C_j|\cdot\bigg\|\bmu_j-\frac{1}{|C_j|}\sum_{l\in C_j}\bm{x}_l\bigg\|^2
\]
shows that the empirical centroid minimizes the left-hand side.
Consequently, after scaling the objective by $1/n$, problem~\eqref{problem: balanced kmeans} is equivalent to
\begin{equation}
\label{problem: original_balanced}
\begin{aligned}
    & \text{minimize}
    &&\quad f(\bmF,\bmu):=\sum_{i\in[n]}\sum_{j\in[k]}F_{ij}\|\bm{x}_i-\bmu_j\|^2 \\[5pt]
    & \text{subject to}
    &&\quad F_{ij}=\frac{1}{n}\bm{1}_{\{i\in C_j\}}, \qquad C_1\sqcup\cdots\sqcup C_k=[n], \qquad |C_j|=\frac{n}{k} \quad \forall\,j\in[k].
\end{aligned}
\end{equation}
We call $\bmF$ an \textit{integral coupling} if $F_{ij}\in\{0,1/n\}$ for every $i\in[n]$ and $j\in[k]$, as in the first constraint of~\eqref{problem: original_balanced}. 
Thus, the feasible matrices above are precisely the integral couplings whose row sums are $1/n$ and whose column sums are $1/k$.

We relax the discrete constraints in \eqref{problem: original_balanced} to the transportation polytope
\[
\bUnk:=\bigg\{\bmF\in\mathbb{R}^{n\times k}_{\geq0}:\bmF\bm{1}_k=\frac{1}{n}\bm{1}_n,\quad \bmF^T\bm{1}_n=\frac{1}{k}\bm{1}_k\bigg\}.
\]
The row constraints distribute the mass $1/n$ of each data point among the cluster labels, while the column constraints assign total mass $1/k$ to every cluster label.
This gives the biconvex problem
\begin{equation}
\label{prob:main_obj}
\text{minimize}\quad f(\bmF,\bmu)
\quad\text{subject to}\quad \bmF\in\bUnk,\quad \bmu\in\Rdk.
\end{equation}
When $k$ divides $n$, the extreme points of $\bUnk$ are precisely the integral couplings; see Appendix~\ref{sec:appendix}.
As a result, the relaxation from \eqref{problem: original_balanced} to \eqref{prob:main_obj} is tight.

For context, given probability vectors $\bm{a}\in\mathbb{R}^m_{\geq0}$ and $\bm{b}\in\mathbb{R}^{\ell}_{\geq0}$ and a cost matrix $\bm{C}\in\mathbb{R}^{m\times\ell}$, the \textit{discrete Kantorovich problem} is
\[
\begin{aligned}
&\text{minimize} && \quad \sum_{i\in[m]}\sum_{j\in[\ell]}C_{ij}F_{ij}\\
&\text{subject to} && \quad \bmF\bm{1}_{\ell}=\bm{a},\quad \bmF^T\bm{1}_m=\bm{b},\quad \bmF\geq0.
\end{aligned}
\]
This is the linear-programming relaxation of the corresponding \textit{discrete Monge problem}, in which the mass at each source point in $[m]$ must be sent to a single target point in $[\ell]$.
See Chapter~2 of \cite{Peyr2018ComputationalOT} for further background.
In our setting, $\bm{a}=\frac{1}{n}\bm{1}_n$, $\bm{b}=\frac{1}{k}\bm{1}_k$, and $C_{ij}=\|\bm{x}_i-\bmu_j\|^2$.

By the previous paragraph, minimizing $f(\cdot,\bmu)$ for a fixed choice of $\bmu\in\Rdk$ is a well-studied discrete Kantorovich problem.
On the other hand, minimizing $f(\bmF,\cdot)$ for a fixed choice of $\bmF\in\bUnk$ is an unconstrained convex problem.
Differentiating with respect to the centroids gives that the unique minimizer is
\[
\bmu=k\bm{X}\bmF,
\qquad
\bm{X}:=[\bm{x}_1\ \cdots\ \bm{x}_n]\in\mathbb{R}^{d\times n}.
\]
This suggests the following alternating process.
Initialize with $k$ centroids $\bmu^0\in\Rdk$, and then for each iteration $t=0,1,\ldots$, compute
\begin{equation}
\label{eq: BalLOT iterations}
\bmF^{t+1}\in\arg\min_{\bmF\in\bUnk}f(\bmF,\bmu^t),
\qquad
\bmu^{t+1}:=k\bm{X}\bmF^{t+1}.
\end{equation}
We terminate once $\|\bmu^{t+1}-\bmu^t\|_F$ is below a prescribed tolerance.
We refer to this approach as \textbf{Balanced Lloyd with Optimal Transport (BalLOT)}.


We conclude this subsection by emphasizing the computational advantages of BalLOT.
For alternating approaches for clustering, at each iteration, one must
\begin{itemize}
\item[1.]
compute an $n\times k$ cost matrix, which takes $O(nkd)$ operations,
\item[2.]
perform an assignment step, the cost of which depends on the method, and then
\item[3.]
compute new centroids, which takes $O(nd)$ operations.
\end{itemize}
For the assignment step, Lloyd's algorithm takes $O(nk)$ operations, meaning the per-iteration computation is dominated by the cost matrix step, but does not enforce balance.
Meanwhile, replicated-centroid matching solves an $n\times n$ assignment problem, which by the Hungarian algorithm costs $O(n^3)$ operations.
By contrast, the BalLOT assignment is a rectangular transportation problem with $nk$ variables and $n+k$ marginal constraints.
The exact runtime depends on the transportation solver, but the rectangular formulation is computationally attractive when $k\ll n$.
Empirically, we observe that the per-iteration cost of BalLOT scales roughly linearly in $n$ for fixed values of $d$ and $k$.

\subsection{Related work}

BalLOT is not entirely unprecedented as an approach to balanced clustering.
The idea of formulating cluster-size constraints as a linear program dates back at least 26~years to \cite{Ng2000ConstrainedKMeans,bradley2000constrained}.
When specializing the cluster sizes to be equal, their constrained $k$-means heuristic reduces to BalLOT.
Building on this idea, \cite{ZHU2010883}~designed a heuristic clustering algorithm using an integer linear program to enforce desired cluster sizes.
Similar linear-programming strategies also appear in political redistricting~\cite{CohenAddad2017Balanced}, where voting districts with equal-sized populations are required by law.
While these precursors propose linear programming approaches to constrained $k$-means clustering, they do not offer theoretical justification for the observed algorithmic performance.

In this paper, we prove that BalLOT is performant on data that was drawn from certain random models.
This follows a growing line of work that uses \textit{mixture models} to theoretically evaluate clustering algorithms.
For example, such models have been used to study Lloyd's algorithm~\cite{lu2016statistical,awasthi2018clustering}, linear-programming relaxations of $k$-means and $k$-median~\cite{nellore2015recovery,awasthi2015relax,DeRosaKhajavirad2022RatioCut,DelPiaMa2023KMedian}, the Peng--Wei SDP~\cite{awasthi2015relax,iguchi2017probably,mixon2017clustering,li2020when} and the Amini--Levina SDP~\cite{fei2022hidden}.
Of these methods, the Amini--Levina SDP is the only one specifically designed for balanced clustering.
In particular, there is presently a dearth of theory for balanced clustering algorithms, especially for algorithms with reasonable runtimes.

Perhaps the most popular model of recent for analyzing the performance of clustering algorithms is the \textit{(balanced) stochastic ball model}.
We will define this more carefully in the following section, but for now, suffice it to say that $n/k$ points are drawn at random from each of $k$ unit balls, where the minimum separation between ball centers is denoted by $\Delta$.
This particular mixture model is attractive since it is information-theoretically feasible to recover the ``intended'' clustering once $\Delta>2$.
As a point of comparison, Table~\ref{table.separations} displays the best known performance guarantees that have been proven for various algorithms.
(While the $k$-means SDP and LP listed here were not intended for the balanced version of $k$-means clustering, the dual certificates used to prove these guarantees can be transferred to their balanced counterparts.)
Judging by this comparison, our BalLOT result is competitive with the literature.

\begin{table}[t]
\centering
\begin{tabular}{llll}\hline
Algorithm & Separation &  No.\ Clusters \qquad\qquad & Reference \\ \hline\hline
thresholding \qquad\qquad & $\Delta>4$ & any & folklore \\ \hline
SDP & $\Delta > 2\sqrt{2}\cdot(1+1/\sqrt{d})$ \qquad\qquad & any & Theorem~3 in \cite{awasthi2015relax} \\
& $\Delta > 2+k^2/d$ & any & Theorem~9 in \cite{iguchi2017probably} \\
& $\Delta > 2+\Omega(\sqrt{k/d})$ & any & Corollary 2 in \cite{li2020when}\\\hline
LP & $\Delta > 1+\sqrt{3}$ & $k=2$ & Theorem~14 in \cite{DeRosaKhajavirad2022RatioCut} \\\hline
BalLOT & $\Delta>2\sqrt{2}$ & $k=2$ & Corollary~\ref{cor:diameter-initialization-stochastic} \\\hline
\end{tabular}
\caption{Balanced $k$-means clustering guarantees under the balanced stochastic ball model.}
\label{table.separations}
\end{table}

\subsection{Contributions and roadmap}

In the next section, we present our theoretical results.
We start in Subsection~\ref{sec: algorithmic guarantees}, which provides some basic regularity guarantees for BalLOT, namely that every BalLOT assignment step has a unique integral minimizer, provided the data is generic (Theorem~\ref{thm.BalLOT is well defined}), and that the usual positive centroid-update stopping rule terminates after finitely many iterations (Theorem~\ref{thm.finite termination}).
Next, Subsection~\ref{sec: global average landscape} defines the balanced mixture model we use in this paper before presenting one of our main results (Theorem~\ref{thm: informal_average_landscape}) that every local minimizer of the average BalLOT objective is globally optimal, and that every global minimizer recovers the ground truth centers up to permutation.
Subsection~\ref{sec: one-step recovery under SBM} then presents guarantees for the \textit{first iteration} of BalLOT.
In particular, we \textit{exactly} recover well separated planted clusters (Theorem~\ref{thm:informal_basin_of_attraction}), and we \textit{partially} recover less separated planted clusters when $k=2$ (Theorem~\ref{thm.k=2 misclustering rate}) and when $k>2$ (Theorem~\ref{thm:misclutering ratio}).
Finally, Subsection~\ref{sec: initialization schemes} supplies initialization schemes (Theorem~\ref{thm:informal_choices of initialization}) that suffice for the one-step recovery guarantees in Subsection~\ref{sec: one-step recovery under SBM}.

After stating these theoretical results in Section~\ref{sec: main results}, we prove them in Section~\ref{sec: formal statements}.
Then we present numerical experiments in Section~\ref{sec: numerical experiments} that corroborate our results before concluding in Section~\ref{sec.discussion} with a discussion.



\section{Main results}\label{sec: main results}

\subsection{Regularity guarantees}
\label{sec: algorithmic guarantees}

The first result shows that, under weak genericity conditions, BalLOT returns a balanced clustering at every assignment step.
We say that the data $\bm{X}$ is \textit{generic} if it avoids a proper algebraic set.
For example, data drawn from an absolutely continuous probability distribution is generic almost surely.

\begin{theorem}
\label{thm.BalLOT is well defined}
For generic data $\bm{X}$, suppose the columns of the initialization $\bmu^0$ are distinct columns of $\bm{X}$, or suppose $\bmu^0$ is a generic member of $\Rdk$.
Then for every BalLOT iteration $t=0,1,\ldots$, the minimizer of $f(\bmF,\bmu^t)$ subject to $\bmF\in\bUnk$ is unique and integral.
\end{theorem}

The next result justifies the stopping rule in \eqref{eq: BalLOT iterations}.

\begin{theorem}
\label{thm.finite termination}
Regardless of the initialization $\bmu^0$, one has
\[
\|\bmu^{t+1}-\bmu^t\|_F\longrightarrow0.
\]
Consequently, for every tolerance $\varepsilon>0$, the BalLOT stopping rule $\|\bmu^{t+1}-\bmu^t\|_F<\varepsilon$ terminates after finitely many iterations.
\end{theorem}

See Subsection~\ref{sec: well-defineness} for the proofs.

\subsection{Balanced mixture models and the average landscape}
\label{sec: global average landscape}

We next define the random data models used in our theoretical guarantees.

\begin{definition}[balanced mixture model]
\label{def.balanced mixture model}
Given distinct ground truth centers $\hmu_1,\ldots,\hmu_k\in\mathbb{R}^d$, consider the data points
\[
\bm{x}_i=\hmu_{\sigma(i)}+\bm{g}_i,
\qquad i\in[n],
\]
where $\sigma\colon[n]\to[k]$ is the ground truth cluster assignment and $\bm{g}_1,\ldots,\bm{g}_n$ are independent realizations of a mean-zero random vector $\bm{g}\in\mathbb{R}^d$.
We say the model is \textit{balanced} if $|\sigma^{-1}(j)|=n/k$ for every $j\in[k]$.
The partition $\{ \sigma^{-1}(j)\}_{j\in[k]}$ of $[n]$ then forms the \textit{planted clustering}.
We write
\[
\Delta:=\min_{\substack{a,b\in[k]\\a\neq b}}\|\hmu_a-\hmu_b\|
\]
for the minimum separation between distinct ground truth centers.
\end{definition}

A \textit{balanced stochastic ball model} is a balanced mixture model in which $\bm{g}$ has a rotationally invariant distribution supported in the unit Euclidean ball.
A \textit{balanced Gaussian mixture model} is the special case $\bm{g}\sim N(\bm{0},\bm{\Sigma})$, and a \textit{balanced sub-Gaussian mixture model} is the special case in which $\bm{g}$ is mean zero and sub-Gaussian.

Although $f$ is biconvex rather than convex, for every balanced mixture model, the following result reports that the \textit{average} objective $\mathbb{E}f$ is well suited for local optimization; here, the expectation is taken with respect to the independent noise vectors $\bm{g}_1,\ldots,\bm{g}_n$.

\begin{theorem}
\label{thm: informal_average_landscape}
Under a balanced mixture model with $\Delta>0$ and $\mathbb{E}\|\bm{g}\|^2<\infty$, every local minimizer of $\mathbb{E}f(\bmF,\bmu)$ subject to $\bmF\in\bUnk$ and $\bmu\in\Rdk$ is a global minimizer.
Moreover, the global minimizers are precisely the planted solutions up to a permutation of the cluster labels.
\end{theorem}

See Subsection~\ref{sec: average_landscape} for the proof.
We interpret Theorem~\ref{thm: informal_average_landscape} as a global characterization of the average objective, rather than as a finite-sample recovery theorem.
Since BalLOT does not directly operate on the average objective, we next give local finite-sample guarantees under the stochastic ball model.

\subsection{One-step recovery}
\label{sec: one-step recovery under SBM}

BalLOT is initialized by $\bmu^0$, and its first assignment update is $\bmF^1$.
In some cases, this assignment already recovers the planted clustering.

\begin{definition}[one-step recovery]
\label{def: one-step recovery}
Suppose the data have planted assignment $\sigma\colon[n]\to[k]$.
We say that BalLOT achieves \textit{one-step recovery} if $\bmF^1$ is the integral coupling induced by the planted clustering, up to a permutation of the cluster labels.
Equivalently, there exists $\pi\in S_k$ such that
\[
F^1_{ij}=\frac{1}{n}\bm{1}_{\{\pi(\sigma(i))=j\}}
\qquad \forall\,i\in[n],\ j\in[k].
\]
As a result, $\bmu^1$ consists of the empirical centroids of the planted clusters.
\end{definition}

One-step recovery requires two things: (1) a good initialization and (2) well-clustered data.
In this section, we use the geometry of the given data to obtain sufficient conditions for~(1), which we then further develop in Subsection~\ref{sec: initialization schemes}.
For~(2), we require more than just a mixture model; we require a mixture model in which $\|\bm{g}\|$ is almost surely bounded, e.g., in the balanced stochastic ball model.
It turns out that we can neatly factor out much of the probability from our analysis by considering the following weakening of the balanced stochastic ball model:

\begin{definition}[balanced deterministic ball model]
\label{def.deterministic ball model}
Given distinct ground truth centers $\bmu^\natural_1,\ldots,\bmu^\natural_k\in\mathbb{R}^d$, consider the data points
\[
\bm{x}_i=\bmu^\natural_{\sigma(i)}+\bm{v}_i,
\qquad i\in[n],
\]
where $\sigma\colon[n]\to[k]$ is the ground truth cluster assignment and $\bm{v}_1,\ldots,\bm{v}_n\in \mathbb{R}^d$ are arbitrary vectors of norm at most $1$.
We say the model is \textit{balanced} if $|\sigma^{-1}(j)|=n/k$ for every $j\in[k]$.
The partition $\{ \sigma^{-1}(j)\}_{j\in[k]}$ of $[n]$ then forms the \textit{planted clustering}.
We write
\[
\Delta:=\min_{\substack{a,b\in[k]\\a\neq b}}\|\bmu^\natural_a-\bmu^\natural_b\|
\]
for the minimum separation between distinct ground truth centers.
\end{definition}


In what follows, we supply a one-step recovery guarantee under the balanced deterministic ball model before using it to prove results for data arising from the balanced stochastic ball model.

\begin{theorem}
\label{thm:informal_basin_of_attraction}
Given data from the balanced deterministic ball model, initialize BalLOT with any choice of centroids $\{\bmu_j^0\}_{j\in[k]}$ that, after relabeling, satisfy
\begin{equation}
\label{eq.weak bound}
\max_{j\in[k]}\|\bmu_j^0-\bmu_j^\natural\|<\left\{\begin{array}{rl}
\sqrt{(\frac{\Delta}{2})^2-1} &\text{ if } k=2\\[0.7em]
\frac{\Delta}{2}-1 &\text{ if } k>2.
\end{array}\right.
\end{equation}
Then BalLOT achieves one-step recovery.
\end{theorem}

See Subsection~\ref{sec: basin of attraction} for a proof of Theorem~\ref{thm:informal_basin_of_attraction}.

The $k>2$ case of condition~\eqref{eq.weak bound} also guarantees one-step recovery by Lloyd's algorithm.
Indeed, if $\|\bm{x}-\bmu_a^\natural\|\leq1$ and $\delta:=\max_j\|\bmu_j^0-\bmu_j^\natural\|<\Delta/2-1$, then for each $b\neq a$, we have
\[
\|\bm{x}-\bmu_a^0\|
\leq\|\bm{x}-\bmu_a^\natural\|+\|\bmu_a^\natural-\bmu_a^0\|
\leq 1+\delta
<\frac{\Delta}{2}
< \Delta-1-\delta
\leq \|\bm{x}-\bmu_b^0\|,
\]
where the last step is a rearrangement of
\[
\Delta
\leq\|\bmu_a^\natural-\bmu_b^\natural\|
\leq\|\bmu_a^\natural-x\|+\|\bm{x}-\bmu_b^0\|+\|\bmu_b^0-\bmu_b^\natural\|
\leq 1+\|\bm{x}-\bmu_b^0\|+\delta.
\]
Thus, we view the $k>2$ statement as a simple deterministic basin rather than a sharp threshold.
The proof reveals why the condition is loose: cyclical monotonicity produces inequalities along misclustering cycles of length at most $k$, and our argument retains only one nonnegative summand from such a cycle.
This is less of an issue in the $k=2$ case, where cycles reduce to swaps.
The numerical experiments presented in Figure~\ref{fig:heat_map_2_means} illustrate that our bound in the $k=2$ case appears to be sharp, while the bound is already loose in the $k=3$ case.

We now transition to theory under the balanced stochastic ball model.
In this setting, one can describe partial recovery of $k=2$ clusters in terms of the angle between the planted and initialized center differences.

\begin{theorem}
\label{thm.k=2 misclustering rate}
Assume $d\geq2$.
Fix $k=2$ ground truth centers $\hmu_1,\hmu_2\in\mathbb{R}^d$ and distinct initial centroids $\bmu_1^0,\bmu_2^0\in\mathbb{R}^d$.
Denote the planted distance and the initialization's cosine similarity by
\[
\Delta:=\|\hmu_1-\hmu_2\|,
\qquad
\cos\theta:=\bigg|\bigg\langle
\frac{\bmu_1^0-\bmu_2^0}{\|\bmu_1^0-\bmu_2^0\|},
\frac{\hmu_1-\hmu_2}{\|\hmu_1-\hmu_2\|}
\bigg\rangle\bigg|,
\]
and consider data points drawn from the balanced stochastic ball model.
\begin{itemize}
\item[(a)] If $\Delta\cos\theta\geq2$, then almost surely, BalLOT achieves one-step recovery.
\item[(b)] If $\Delta\cos\theta<2$, then for every $\varepsilon\in(0,1)$, with probability at least $1-\varepsilon$, the first BalLOT update satisfies
\[
\min_{\pi\in S_2}\frac{|\{i\in[n]:\sigma^1(i)\neq\pi(\sigma(i))\}|}{n}
\leq
\sqrt{
\exp\bigg(-\frac{d-1}{4}(\Delta\cos\theta)^2\bigg)
+
\sqrt{\frac{1}{n}\log\bigg(\frac{n}{2\varepsilon}\bigg)}
}.
\]
Here, $\sigma^1\colon[n]\to[2]$ denotes the assignment determined by $\bmF^1$.
\end{itemize}
\end{theorem}

See Figure~\ref{fig: SNR_0 and R_1} for an illustration of the threshold phenomenon at $\Delta\cos\theta=2$ described in Theorem~\ref{thm.k=2 misclustering rate}.
We managed to generalize Theorem~\ref{thm.k=2 misclustering rate} to the $k>2$ case, though we only provide a qualitative statement in this case, since the explicit version obfuscates the forest with its trees.



\begin{theorem}
\label{thm:misclutering ratio}
Fix an initialization $\bmu^0$, and consider data points drawn from the balanced stochastic ball model.
With high probability, the misclustering ratio of the first BalLOT update $\bmF^1$ is bounded above by some function of  $k$, $n$, $d$, $\Delta$, and the distance between $\bmu^0$ and the planted cluster means $\hmu$. 
In particular, this upper bound is smaller when $n$, $d$, and $\Delta$ are larger, and when $\bmu^0$ is closer to $\hmu$.
\end{theorem}


See Subsection~\ref{sec: probabilistic conditions} for the proofs of Theorems~\ref{thm.k=2 misclustering rate} and~\ref{thm:misclutering ratio}.

\subsection{Initialization schemes}
\label{sec: initialization schemes}

Theorem~\ref{thm:informal_basin_of_attraction} motivates initialization schemes that produce centroids close to the planted centers.
For each initialization, we take $\bmu^0$ to consist of particular columns from the input data $\bm{X}=[\bm{x}_1\cdots\bm{x}_n]$.
In the case where $k=2$, we consider the \textit{diameter initialization}, which initializes BalLOT at any pair $(\bm{x}_a,\bm{x}_b)$ such that
\[
\|\bm{x}_a-\bm{x}_b\|
=
\max_{i,i'\in[n]}\|\bm{x}_i-\bm{x}_{i'}\|.
\]
For general $k$, we use the \textit{$m$-coupon-collecting initialization}, which first draws $m$ column vectors from~$\bm{X}$ uniformly without replacement.
Two of these points are declared adjacent if their distance is at most $\min\{\Delta-2,2\}$.
If this graph is a disjoint union of $k$ cliques, then we initialize BalLOT at one representative from each clique (otherwise, the initialization fails outright).


\begin{theorem}
\label{thm:informal_choices of initialization}\
\begin{itemize}
\item[(a)] 
Given a balanced deterministic ball model with $k=2$ and $\Delta>2\sqrt{2}$ such that
\[
\sum_{\sigma(i)=j}\bm{v}_i=\bm{0}
\qquad
\forall\,j\in[k],
\]
the diameter initialization results in one-step recovery.
\item[(b)] 
Given a balanced stochastic ball model with $k\geq 2$ and $\Delta>2$ such that
\[
m:=\bigg\lceil k\log\bigg(\frac{2k}{\varepsilon}\bigg)\bigg\rceil\leq n
\qquad\text{and}\qquad
\mathbb{E}\|\bm{g}\|^2
\leq
\frac{\varepsilon}{2m}\bigg(\frac{\Delta}{2}-1\bigg)^2,
\]
the $m$-coupon-collecting initialization results in one-step recovery with probability $\geq1-\varepsilon$.
\end{itemize}
\end{theorem}

For part~(a) above, we do not know if the condition $\Delta>2\sqrt{2}$ is sharp, but we do know that $\Delta>2$ alone does not suffice; see Appendix~\ref{sec: diameter counterexample} for an explicit counterexample.
For part~(b), the probability is on both the data model and the randomness in the initialization.
This second guarantee requires the within-cluster variance to decay with $k$.
In practice, we are inclined to use the $k$-means++ initialization, but we leave a theoretical analysis of this initialization for future work.

We conclude this section with a consequence of Theorem~\ref{thm:informal_choices of initialization} that allows us to compare with related literature in Table~\ref{table.separations}.

\begin{corollary}
\label{cor:diameter-initialization-stochastic}
Given a balanced stochastic ball model with $k=2$ and $\Delta>2\sqrt{2}$, the diameter initialization results in one-step recovery with probability $1-e^{-\Omega_\Delta(n)}$.
\end{corollary}

The proofs of these results appear in Subsection~\ref{sec:initialization}.

\section{Proofs}
\label{sec: formal statements}

In this section, we prove the results stated in Section~\ref{sec: main results}.

\subsection{Proofs of Theorems~\ref{thm.BalLOT is well defined} and~\ref{thm.finite termination}}
\label{sec: well-defineness}

\begin{proof}[Proof of Theorem~\ref{thm.finite termination}]
For each $t\in\mathbb{N}$, the centroid update gives
\[
f(\bmF^{t+1},\bmu^t)
=\sum_{i\in[n]}\sum_{j\in[k]}F_{ij}^{t+1}\|\bm{x}_i-\bmu_j^t\|^2
=f(\bmF^{t+1},\bmu^{t+1})+\frac{1}{k}\|\bmu^{t+1}-\bmu^t\|_F^2.
\]
Indeed, after expanding the square, the cross term vanishes because $\bmF^{t+1}\in\bUnk$ and $\bmu^{t+1}=k\bm{X}\bmF^{t+1}$.
Since $\bmF^{t+2}$ minimizes $f(\cdot,\bmu^{t+1})$ over $\bUnk$,
\[
f(\bmF^{t+1},\bmu^t)
\geq
f(\bmF^{t+2},\bmu^{t+1})
+\frac{1}{k}\|\bmu^{t+1}-\bmu^t\|_F^2.
\]
Summing and telescoping gives
\begin{align*}
\frac{1}{k}\sum_{t=0}^{T-1}\|\bmu^{t+1}-\bmu^t\|_F^2
&\leq
\sum_{t=0}^{T-1}\big(f(\bmF^{t+1},\bmu^t)-f(\bmF^{t+2},\bmu^{t+1})\big)\\
&=f(\bmF^1,\bmu^0)-f(\bmF^{T+1},\bmu^T)\\[1em]
&\leq f(\bmF^1,\bmu^0).
\end{align*}
Therefore,
\[
\sum_{t=0}^{\infty}\|\bmu^{t+1}-\bmu^t\|_F^2<\infty,
\]
which implies $\|\bmu^{t+1}-\bmu^t\|_F\to0$.
The claimed finite termination follows for every positive tolerance.
\end{proof}

We next prove the generic uniqueness statement.
Theorem~\ref{thm: generic k points initialization} treats the first assignment step when the initialization consists of data points, while Theorem~\ref{thm.generic initialization by partition} treats every subsequent assignment step after an integral update.
Integrality will then follow separately from the vertex characterization in Appendix~\ref{sec:appendix}.

\begin{lemma}[Generic centroid initialization]
\label{lem: generic centroid initialization}
For generic data $\bm{X}$ and generic $\bmu^0\in\Rdk$, the minimizer of $f(\bmF,\bmu^0)$ over $\bUnk$ is unique.
\end{lemma}

\begin{proof}
The polytope $\bUnk$ has finitely many vertices, all of which are integral by Fact~\ref{fact:extreme points of Unk} in Appendix~\ref{sec:appendix}.
For two distinct vertices $\bm{P}^1$ and $\bm{P}^2$, the equality
\[
f(\bm{P}^1,\bmu)-f(\bm{P}^2,\bmu)=0
\]
is a linear equation in the entries of $\bmu$, since the common row and column sums cancel the terms that are quadratic in $\bmu$ and independent of $\bmu$.
Since $\bm{P}^1\neq\bm{P}^2$, at least one coefficient vector $\sum_i(P^1_{ij}-P^2_{ij})\bm{x}_i$ is a nonzero formal linear combination of the data, and hence is nonzero for generic data.
Thus, the displayed equation is not identically zero in $\bmu$, and the centroid matrices for which two vertices have equal objective value lie in a finite union of proper hyperplanes.
Outside this union, all vertex objective values are distinct, and so the linear program has a unique minimizer.
\end{proof}
\begin{theorem}[Initialization by $k$ data points]
Suppose $\{\bm{x}_i\}_{i\in[n]}$ is generic and $\bmu^0$ consists of $k$ distinct choices of $\bm{x}_{i}$.
Then the minimizer of $f(\bmF,\bmu^0)$ subject to $\bmF\in\bUnk$ is unique. 
    \label{thm: generic k points initialization}
\end{theorem}

\begin{proof}
Suppose the minimizer is not unique.  
Then the optimal face contains two distinct vertices $\bm{P}^1,\bm{P}^2\in\bUnk$.  
By Fact~\ref{fact:extreme points of Unk} in Appendix~\ref{sec:appendix}, these vertices are integral, and
\begin{equation}
\label{eq.equal f vals}
f(\bm{P}^1,\bmu^0)=f(\bm{P}^2,\bmu^0).
\end{equation}
We will show that $\bm{P}^1$ and $\bm{P}^2$ are suboptimal for $f(\cdot,\bmu^0)$, which gives the desired contradiction.

It will be convenient to reformulate \eqref{eq.equal f vals}, but this requires additional notation.
Let $K\subseteq[n]$ denote the set of indices $i$ for which $\bm{x}_i$ is one of the $k$ vectors in $\bmu^0$, and let $\alpha_1,\alpha_2\colon [n]\to K$ correspond to the cluster assignment functions associated with $\bm{P}^1$ and $\bm{P}^2$, respectively.
In particular, $\bm{x}_i$ is assigned by $\bm{P}^\ell$ to $\bm{x}_{\alpha_\ell(i)}$, which in turn is one of the vectors in $\bmu^0$.

Then \eqref{eq.equal f vals} is equivalent to our data $\bm{X}:=\{\bm{x}_i\}_{i\in[n]}$ satisfying $p(\bm{X})=0$, where
\[
p(\bm{X}) 
:= \sum_{i\in [n]} \la \bm{x}_i, \bm{x}_{\alpha_1(i)} - \bm{x}_{\alpha_2(i)} \ra.
\]
Since $\bm{X}$ is generic, it follows that $p$ is identically zero.
We will leverage cancellations in the formal polynomial $p(\bm{X})$ to infer a way to decrease $f(\cdot,\bmu^0)$, thereby demonstrating the claimed suboptimality of $\bm{P}^1$ and $\bm{P}^2$.

For each $i\not\in K$, the only appearance of $\bm{x}_i$ in $p(\bm{X})$ is in the linear term $\la \bm{x}_i, \bm{x}_{\alpha_1(i)} - \bm{x}_{\alpha_2(i)} \ra$.
Since $p(\bm{X})$ is identically zero, this term must be identically zero, too, i.e., $\alpha_1(i)=\alpha_2(i)$. 
Rearranging $p(\bm{X})=0$ then gives
\[
\sum_{i\in K} \la \bm{x}_i, \bm{x}_{\alpha_1(i)} \ra
=\sum_{i\in K} \la \bm{x}_i,  \bm{x}_{\alpha_2(i)} \ra.
\]
Equating terms, then for each $i\in K$, either $\alpha_1(i)=\alpha_2(i)$ or both $\alpha_1(i)=j$ and $\alpha_2(j)=i$.
In the latter case, we may assume $j\neq i$, since otherwise $\alpha_1(i)=i=\alpha_2(i)$.
Let $K'$ denote the subset of $K$ for which $\alpha_1(i)\neq\alpha_2(i)$.
Notably, $K'$ is nonempty since $\bm{P}^1\neq\bm{P}^2$ by assumption.
Furthermore, $\alpha_1$ and $\alpha_2$ induce inverse derangements of $K'$.
As such, $\bm{P}^1$ assigns some $\bm{x}_i$ with $i\in K'$ to a different $\bm{x}_j$ with $j\in K'$, and so we can decrease $f(\cdot,\bmu^0)$ by instead using the identity assignment on $K'$.
\end{proof}

\begin{theorem}[Initialization by partition of data]\label{thm.generic initialization by partition}
Suppose $\{\bm{x}_i\}_{i\in[n]}$ is generic and $\bmu^0$ consists of the centroids of a balanced partition of the $\bm{x}_i$'s.
Then the minimizer of $f(\bmF,\bmu^0)$ subject to $\bmF\in\bUnk$ is unique.
\end{theorem}

\begin{proof}
Suppose the minimizer is not unique.  As in the previous proof, the optimal face then contains distinct integral vertices $\bm{P}^1,\bm{P}^2\in\bUnk$ for which $f(\bm{P}^1,\bmu^0)=f(\bm{P}^2,\bmu^0)$.  It suffices to show that these couplings are suboptimal for $f(\cdot,\bmu^0)$.
Suppose the partition of $[n]$ that determines $\bmu^0$ is given by 
\[
C_1\sqcup \cdots \sqcup C_k 
= [n].
\]
For $\ell\in \{1,2 \}$, we define $\alpha_\ell\colon [n]\to \{ C_1, C_2,\dotsc, C_k\}$ to be the set-valued assignment by $\bm{P}^\ell$, that is,
\[
\alpha_\ell(i) 
= C_j\quad \text{if $\bm{P}^\ell$ assigns $\bm{x}_i$ to the centroid of $\{\bm{x}_{i'}\}_{i'\in C_j}$}.
\] 
Then our assumption $f(\bm{P}^1,\bmu^0)=f(\bm{P}^2,\bmu^0)$ is equivalent to $p(\bm{X})=0$, where
\[
p(\bm{X}) 
= \sum_{i\in \cA} \la \bm{x}_i , \sum_{j\in \alpha_1(i)} \bm{x}_j - \sum_{j' \in \alpha_2(i) } \bm{x}_{j'}  \ra,
\]
and $\cA := \{ i\in [n]: \alpha_1(i) \neq \alpha_2(i)\}$.
As before, since $p(\bm{X})=0$ and $\bm{X}$ is generic, it follows that $p(\bm{X})$ is identically zero.
In particular, the following identity holds when treating the $\bm{x}_i$'s as formal variables:
\[
\sum_{i\in \cA}\sum_{j\in \alpha_1(i)} \la \bm{x}_i ,  \bm{x}_j \ra 
= \sum_{i\in \cA} \sum_{j' \in \alpha_2(i) } \la \bm{x}_i,\bm{x}_{j'}  \ra.
\]
Notably, every term on the left-hand side corresponds to a term on the right-hand side, and vice versa.
Since $\alpha_1(i)$ and $\alpha_2(i)$ are disjoint whenever $\alpha_1(i)\neq\alpha_2(i)$, it follows that
\[
\text{$i\in \cA$ ~and~ $j\in \alpha_1(i)$}
\quad
\Longleftrightarrow
\quad
\text{$j \in \cA$ ~and~ $i\in \alpha_2(j)$}.
\]
We refer to this equivalence as the \textit{exchange property}.
In what follows, we repeatedly appeal to the exchange property in order to uncover how $\alpha_1$ and $\alpha_2$ behave over $\cA$.
(Throughout, we write $i\sim j$ if $i$ and $j$ belong to the same $C_\ell$.)

First, we claim that $\cA$ is a union of $C_\ell$'s. 
Pick any $i\in\cA$ and any $i'\sim i$.
If we select $j\in\alpha_1(i)$, then by the exchange property (in the forward direction), we have $j\in\cA$ and $i\in\alpha_2(j)$.
Since $i'\sim i$, we also have $i'\in\alpha_2(j)$.
Since $j\in\cA$ and $i'\in\alpha_2(j)$, the exchange property (in the reverse direction) then gives that $i'\in\cA$.

Next, we claim that for each $\ell\in\{1,2\}$, $\alpha_\ell$ maps points in $\cA$ to points in $\cA$.
The $\ell=1$ case follows from the exchange property in the forward direction, while the $\ell=2$ case follows from the exchange property in the reverse direction.

Next, we claim that $i\sim i'$ in $\cA$ implies $\alpha_\ell(i)=\alpha_\ell(i')$ for both $\ell\in\{1,2\}$, i.e., each $\alpha_\ell$ maps clusters in $\cA$ to clusters in $\cA$.
We will prove this for $\ell=1$, as the proof for $\ell=2$ is identical.
For $i\sim i'$ in $\cA$, the exchange property (in the forward direction) gives that for every $j\in\alpha_1(i)$ and $j'\in\alpha_1(i')$, it holds that $j,j'\in\cA$ and $i\in\alpha_2(j)$ and $i'\in\alpha_2(j')$, in which case $i\sim i'$ forces $\alpha_2(j)=\alpha_2(j')$.
So we have $j\in\cA$ and $i,i'\in \alpha_2(j)$.
Then the exchange property (in the reverse direction) gives that $j$ is in both $\alpha_1(i)$ and $\alpha_1(i')$, and so $\alpha_1(i)=\alpha_1(i')$.

Finally, we claim that for each $\ell\in\{1,2\}$, if $i,i'\in \cA$ and $\alpha_\ell(i) = \alpha_\ell(i')$, then $i\sim i'$, i.e., each $\alpha_\ell$ permutes the clusters in $\cA$.
(Again, we only prove this for $\ell=1$.)
Indeed, take any $j\in \alpha_1(i) = \alpha_1(i')$.
By exchange property (in the forward direction), it follows that $j\in \cA$ and $i,i'\in  \alpha_2(j)$, meaning $i\sim i'$.

At this point, we know that $\alpha_1$ and $\alpha_2$ both permute the clusters in $\cA$. 
Next, since $\bm{X}$ is generic, the cluster centroids are distinct, and so equality in the lower bound
\begin{align*}
\sum_{a\in [k]} \sum_{\substack{i\in\cA\\\alpha_1(i) = C_a}} \| \bm{x}_i - \bmu^0_a\|^2 
&= \sum_{a\in [k]}
\sum_{\substack{i\in\cA\\\alpha_1(i) = C_a}} \Bigg(
    \bigg\|\bm{x}_i - \frac{1}{n/k} \sum_{\substack{j\in\cA\\\alpha_1(j) = C_a}} \bm{x}_j \bigg\|^2 + \bigg\|\bmu^0_{a} -\frac{1}{n/k} \sum_{\substack{j\in\cA\\\alpha_1(j) = C_a}} \bm{x}_j \bigg\|^2 \Bigg)
    \\
    & \ge \sum_{a\in [k]} \sum_{\substack{i\in\cA\\\alpha_1(i) = C_a}} \bigg\|\bm{x}_i - \frac{1}{n/k} \sum_{\substack{j\in\cA\\\alpha_1(j) = C_a}} \bm{x}_j \bigg\|^2
\end{align*} 
occurs precisely when every $\bmu^0_a$ is the centroid of points indexed by $C_a$, i.e., $\alpha_1$ maps every cluster in $\cA$ to itself. 
Similarly, the same holds for $\alpha_2$. 
Since $\bm{P}^1\neq \bm{P}^2$, it necessarily holds that $\alpha_1\neq \alpha_2$ on $\cA$, and so we can decrease $f(\cdot, \bmu^0)$ by changing one of them to the identity cluster permutation on $\cA$.
\end{proof}

\begin{proof}[Proof of Theorem~\ref{thm.BalLOT is well defined}]
Suppose first that $\bmu^0$ consists of $k$ distinct columns of $\bm{X}$.
Theorem~\ref{thm: generic k points initialization} gives a unique first-step minimizer.
A linear functional on the compact polytope $\bUnk$ has an optimal vertex, and every vertex is integral by Fact~\ref{fact:extreme points of Unk} in Appendix~\ref{sec:appendix}; uniqueness therefore forces $\bmF^1$ to be integral.
Consequently, $\bmu^1=k\bm{X}\bmF^1$ consists of the centroids of a balanced partition.
Theorem~\ref{thm.generic initialization by partition} gives a unique second-step minimizer, which is again integral by the same vertex argument.
Induction proves the claim for every iteration.

If $\bmu^0$ is generic, Lemma~\ref{lem: generic centroid initialization} replaces Theorem~\ref{thm: generic k points initialization} in the first step, and the same induction applies thereafter.
\end{proof}

\subsection{Proof of Theorem~\ref{thm: informal_average_landscape}}
\label{sec: average_landscape}

Put $\hmu:=[\hmu_1\ \cdots\ \hmu_k]\in\mathbb{R}^{d\times k}$.
For $\bmF\in\bUnk$, define the aggregated coupling $\Pi(\bmF)\in\mathbb{R}^{k\times k}$ by
\begin{equation}
\label{eq: definition of Pi}
\Pi(\bmF)_{pj}:=\sum_{\substack{i\in[n]\\\sigma(i)=p}}F_{ij},
\qquad p,j\in[k].
\end{equation}
A direct calculation gives
\begin{equation}
\label{eq: average objective}
\mathbb{E}f(\bmF,\bmu)
=
\sum_{p\in[k]}\sum_{j\in[k]}\Pi(\bmF)_{pj}\|\bmu_j-\hmu_p\|^2
+\mathbb{E}\|\bm{g}\|^2.
\end{equation}
The second term is constant.

We first verify that
\[
\Pi(\bUnk)=\bUkk,
\]
where $\bUkk$ denotes the set of nonnegative $k\times k$ matrices whose row and column sums are all equal to $1/k$.
Indeed, for $\bmF\in\bUnk$, every row and column sum of $\Pi(\bmF)$ equals $1/k$:
\[
\sum_{p\in[k]}\Pi(\bmF)_{pj}
=\sum_{i\in[n]}F_{ij}=\frac{1}{k},
\qquad
\sum_{j\in[k]}\Pi(\bmF)_{pj}
=\sum_{\substack{i\in[n]\\\sigma(i)=p}}\frac{1}{n}=\frac{1}{k}.
\]
Thus, $\Pi(\bUnk)\subseteq\bUkk$.
Conversely, for $\bmpi\in\bUkk$, define
\[
(\bmF_{\bmpi})_{ij}:=\frac{\pi_{pj}}{n/k}
\qquad\text{when}\qquad
\sigma(i)=p.
\]
Then $\bmF_{\bmpi}\in\bUnk$ and $\Pi(\bmF_{\bmpi})=\bmpi$, proving the reverse containment.

Let $(\bm{F}^*, \bmu^*)$ be a local minimizer of $\mathbb{E}f$, and put $\bmpi^*:=\Pi(\bmF^*)$.
For fixed $\bmu^*$, local minimality in the convex variable $\bmF$ implies global minimality: for every $\bmF\in\bUnk$ and sufficiently small $s>0$, the point $(1-s)\bmF^*+s\bmF$ is feasible and arbitrarily close to $\bmF^*$, and linearity in $\bmF$ gives
\[
\mathbb{E}f(\bmF^*,\bmu^*)\leq\mathbb{E}f(\bmF,\bmu^*).
\]
Equivalently, $\bmpi^*$ minimizes
\[
h(\bmpi,\bmu^*):=\sum_{p\in[k]}\sum_{j\in[k]}\pi_{pj}\|\bmu_j^*-\hmu_p\|^2
\]
over $\bUkk$, as $\Pi(\bUnk) = \bUkk$.

By the Birkhoff--von Neumann theorem, we may express $\bmpi^*$ as a convex combination of scaled permutation matrices.
Every scaled permutation appearing with positive weight is also optimal for $h(\cdot,\bmu^*)$.
Fix one such matrix $\bmpi^\dagger$, choose $\bmF^\dagger\in\bUnk$ with $\Pi(\bmF^\dagger)=\bmpi^\dagger$, and put
\[
\bm{P}^s:=(1-s)\bmF^*+s\bmF^\dagger,
\qquad
\bmpi^s:=(1-s)\bmpi^*+s\bmpi^\dagger.
\]
For sufficiently small $s>0$, $(\bm{P}^s,\bmu^*)$ lies in the neighborhood on which $(\bmF^*,\bmu^*)$ is minimal, and it has the same objective value.
It is therefore also a local minimizer.

For fixed $\bmpi\in\bUkk$, the function $h(\bmpi,\cdot)$ is strongly convex, with unique minimizer
\[
\bmu=k\hmu\bmpi.
\]
Applying this to $\bmpi^*$ and $\bmpi^s$ gives
\[
\bmu^*=k\hmu\bmpi^*=k\hmu\bmpi^s.
\]
Since $\bmpi^\dagger$ is an affine combination of $\bmpi^*$ and $\bmpi^s$, it follows that
\[
\bmu^*=k\hmu\bmpi^\dagger.
\]
Because $\bmpi^\dagger$ is a scaled permutation matrix, the columns of $\bmu^*$ are the columns of $\hmu$ in some order.
At these centroids, the minimum value of $h$ is zero.
Since the ground truth centers are distinct, the scaled permutation that attains zero is unique, and therefore $\bmpi^*=\bmpi^\dagger$.
It follows from \eqref{eq: average objective} that $(\bmF^*,\bmu^*)$ is globally optimal.
The same argument also characterizes every global minimizer, proving the final statement of the theorem.

\subsection{Proof of Theorem~\ref{thm:informal_basin_of_attraction}}
\label{sec: basin of attraction}

\begin{lemma}
\label{lem.delta cos theta}
Fix $k=2$ ground truth centers $\bmu^\natural_1,\bmu^\natural_2\in\mathbb{R}^d$ and distinct initial centroids $\bmu_1^0,\bmu_2^0\in\mathbb{R}^d$.
Denote the planted distance and the initialization's cosine similarity by
\[
\Delta:=\|\bmu^{\natural}_1-\bmu^{\natural}_2\|,
\qquad
\cos\theta:=\bigg|\bigg\langle
\frac{\bmu_1^0-\bmu_2^0}{\|\bmu_1^0-\bmu_2^0\|},
\frac{\bmu^{\natural}_1-\bmu^{\natural}_2}{\|\bmu^{\natural}_1-\bmu^{\natural}_2\|}
\bigg\rangle\bigg|.
\]
If $\Delta\cos\theta>2$, then under the balanced deterministic ball model, BalLOT achieves one-step recovery.
\end{lemma}

\begin{proof}
We prove the contrapositive.
Suppose BalLOT does not achieve one-step recovery. Then the balanced clustering $C_1^1\sqcup C_2^1 = [n]$ after one step of BalLOT is distinct from the planted balanced clustering $C_1^\natural \sqcup C_2^\natural$. For convenience, we relabel the initial centroids $\bmu_1^0$ and $\bmu_2^0$ as necessary, so that
\[
\bigg\langle
\frac{\bmu_1^0-\bmu_2^0}{\|\bmu_1^0-\bmu_2^0\|},
\frac{\bmu^{\natural}_1-\bmu^{\natural}_2}{\|\bmu^{\natural}_1-\bmu^{\natural}_2\|}
\bigg\rangle
\geq0,
\]
i.e., the left-hand side is $\cos\theta$, and we index clusters so that $C_1^1$ and $C_2^1$ correspond to $\bmu_1^0$ and $\bmu_2^0$, respectively. 
Select indices $i\in C_1^1\setminus C_1^\natural$ and $j\in C_2^1\setminus C_2^\natural$.
Optimality of the first assignment gives
\[
\|\bm{x}_i-\bmu_1^0\|^2+\|\bm{x}_j-\bmu_2^0\|^2
\leq
\|\bm{x}_i-\bmu_2^0\|^2+\|\bm{x}_j-\bmu_1^0\|^2.
\]
Writing $\bm{x}_i=\bmu^{\natural}_2+\bm{v}_i$ and $\bm{x}_j=\bmu^{\natural}_1+\bm{v}_j$, then rearranging gives
\[
\bigg\langle\bm{v}_i-\bm{v}_j,
\frac{\bmu_1^0-\bmu_2^0}{\|\bmu_1^0-\bmu_2^0\|}\bigg\rangle
\geq\Delta\cos\theta.
\]
By the Cauchy--Schwarz and triangle inequalities, it follows that
\[
\Delta\cos\theta
\leq\|\bm{v}_i-\bm{v}_j\|
\leq\|\bm{v}_i\|+\|\bm{v}_j\|
\leq2.
\qedhere
\]
\end{proof}

\begin{proof}[Proof of Theorem~\ref{thm:informal_basin_of_attraction}]
By assumption, we have a bound on the initialization error:
\[
\delta:=\max_{j\in[k]}\|\bmu_j^0-\bmu^{\natural}_j\|.
\]
Our task is to conclude one-step recovery.

First suppose $k=2$.
Projecting the four points $\bmu^{\natural}_1,\bmu^{\natural}_2,\bmu_1^0,\bmu_2^0$ onto the plane spanned by $\bmu^{\natural}_1-\bmu^{\natural}_2$ and $\bmu_1^0-\bmu_2^0$ preserves the angle between the two difference vectors and does not increase the initialization error.
We may therefore work in two dimensions.
The largest possible angle $\theta$ occurs when the line through $\bmu_1^0$ and $\bmu_2^0$ is tangent to the two radius-$\delta$ initialization balls, as illustrated in Figure~\ref{fig:visualize_k=2}.
The Pythagorean theorem gives
\[
\delta^2=\bigg(\frac{\Delta}{2}\bigg)^2-
\bigg(\frac{\Delta\cos\theta}{2}\bigg)^2.
\]
Since $\delta<\sqrt{(\frac{\Delta}{2})^2-1}$ by assumption, it follows that $\Delta\cos\theta>2$, and so Lemma~\ref{lem.delta cos theta} implies the result.

\begin{figure}[t]
    \centering
    \begin{tikzpicture}[scale =2]
    \draw[dashed] (-3/2,0) circle ({sqrt(5)/2});
    \draw[dashed] (3/2,0) circle ({sqrt(5)/2});
    \draw (-3/2,0) circle (1cm);
    \draw (3/2,0) circle (1cm);
    \fill[black] ({5/6-3/2} ,{sqrt(5)/3}) circle (0.05);
    \fill[black] ({-5/6+3/2} , {-sqrt(5)/3} ) circle (0.05);
    \fill[black]  (-3/2, 0) circle (0.05);
    \fill[black]  (3/2, 0) circle (0.05);
    \fill[black] (0,0) circle (0.05);
    \node[left] at (-3/2,0) {$\bmu^\natural_1$};
    \node[right] at (3/2,0) {$\bmu^\natural_2$};
    \node[above right] at ({5/6-3/2} ,{sqrt(5)/3})  {$\bmu^0_1$};
    \node[below left] at ({-5/6+3/2} ,{-sqrt(5)/3}) {$\bmu^0_2$};

    \coordinate (hmu1) at (-3/2, 0);
    \coordinate (hmu2) at (3/2, 0);
    \coordinate (bmu1) at ({5/6-3/2}, {sqrt(5)/3});
    \coordinate (bmu2) at ({-5/6+3/2}, {-sqrt(5)/3});
    \coordinate (zero) at (0,0);
    \draw (hmu1) -- (bmu1) node[pos=0.3,  above left, xshift=5mm, yshift=2mm] {$ \delta~$};
    \draw (bmu1) -- (bmu2);
    \draw (hmu2) -- (bmu2);
    \draw (hmu1) -- (hmu2);
    \draw (hmu1) -- (zero) node[pos=0.5,  below, xshift=0mm, yshift=-0.05mm]{$ \frac{\Delta}{2}$};
    \draw (bmu1) -- (zero) node[pos=0.6,  above, xshift=2.5mm, yshift=+0.05mm]{$ \frac{\Delta\cos\theta}{2}$};
    \draw pic[draw, angle radius=2mm] {right angle = hmu1--bmu1--bmu2};
    \draw pic[draw, angle radius=2mm] {right angle = bmu1--bmu2--hmu2};
    \draw pic[draw, angle radius = 2mm]
    {angle = bmu1--zero--hmu1};

    \node[above left] at (-0.1,-0.05) {$\theta$};
    \end{tikzpicture}
    \caption{The two solid circles with unit radius are centered at $\bmu^{\natural}_1$ and $\bmu^{\natural}_2$, and these centers have distance $\Delta$. The two dashed circles denote the initialization neighborhood of radius $\delta$ that achieves one-step recovery. By the Pythagorean theorem, the sufficient condition $\Delta\cos\theta > 2$ from Lemma~\ref{lem.delta cos theta} is equivalent to $\delta < ((\frac{\Delta}{2})^2 - 1)^{1/2}$.}
    \label{fig:visualize_k=2}
\end{figure}

Now suppose $k>2$.
Then we are given $\delta<\Delta/2-1$.
Assume toward a contradiction that BalLOT does not achieve one-step recovery.
Then the first-step clustering $C_1^1\sqcup\cdots\sqcup C_k^1=[n]$ is distinct from $C_1^\natural \sqcup\cdots\sqcup C_k^\natural=[n]$.
The assumption on $\delta$ implies that each initialized centroid is closer to its corresponding planted center than to every other planted center: for $a\neq b$,
\[
\|\bmu_a^0-\bmu^{\natural}_a\|
\leq\delta
<\frac{\Delta}{2}
\leq\|\bmu^{\natural}_b-\bmu^{\natural}_a\|-\|\bmu_a^0-\bmu^{\natural}_a\|
\leq\|\bmu_a^0-\bmu^{\natural}_b\|.
\]

Represent the misclustering by a directed multigraph on vertex set $[k]$.
For every $p\in C_a^1\cap C_b^\natural$ with $a\neq b$, include a directed edge $a\to b$ labeled by $p$.
Because both clusterings are balanced, every vertex has equal in-degree and out-degree.
Every finite directed multigraph with this property decomposes into edge-disjoint simple directed cycles.
(This is an Eulerian-digraph argument.)

Fix one cycle from such a decomposition, let its length be $k'\leq k$, list its edge labels as $p_1,\ldots,p_{k'}$, and write the edge labeled $p_i$ as $j_i\to j_{i+1}$, with $j_{k'+1}=j_1$.
Optimality of $\bmF^1$ gives the cyclical monotonicity inequality
\[
\sum_{i\in[k']}\|\bm{x}_{p_i}-\bmu_{j_i}^0\|^2
\leq
\sum_{i\in[k']}\|\bm{x}_{p_i}-\bmu_{j_{i+1}}^0\|^2.
\]
Since $p_i\in C_{j_{i+1}}^\natural$, write $\bm{x}_{p_i}=\bmu^{\natural}_{j_{i+1}}+\bm{g}_{p_i}$.
After rearranging,
\[
\sum_{i\in[k']}
\left[
\langle\bm{g}_{p_i},\bmu_{j_i}^0-\bmu_{j_{i+1}}^0\rangle
-
\frac{1}{2}
\left(
\|\bmu_{j_i}^0-\bmu^\natural_{j_{i+1}}\|^2
-
\|\bmu_{j_i}^0-\bmu^\natural_{j_i}\|^2
\right)
\right]
\geq0.
\]
At least one summand is nonnegative.
For the corresponding ordered pair $a\neq b$ and data point $p$, we have
\begin{equation}
\label{eq: cycle necessary inequality}
\frac{\|\bmu_a^0-\bmu^\natural_b\|^2-\|\bmu_a^0-\bmu^\natural_a\|^2}
{2\|\bmu_a^0-\bmu_b^0\|}
\leq
\bigg\langle\bm{g}_p,
\frac{\bmu_a^0-\bmu_b^0}{\|\bmu_a^0-\bmu_b^0\|}\bigg\rangle
\leq1.
\end{equation}

Put
\[
\bm{v}_a:=\bmu_a^0-\bmu^\natural_a,
\qquad
\bm{v}_b:=\bmu_b^0-\bmu^\natural_b,
\qquad
\bm{w}_{ab}:=\bmu^\natural_a-\bmu^\natural_b.
\]
Then the left-hand side of \eqref{eq: cycle necessary inequality} equals
\[
\frac{2\langle\bm{v}_a,\bm{w}_{ab}\rangle+\|\bm{w}_{ab}\|^2}
{2\|\bm{w}_{ab}+\bm{v}_a-\bm{v}_b\|}.
\]
The numerator is positive because $\|\bm{w}_{ab}\|\geq\Delta>2\delta$.
Writing $x$ and $y$ for the components of $\bm{v}_a$ parallel and orthogonal to $\bm{w}_{ab}$, respectively, and using $\|\bm{v}_b\|\leq\delta$, we obtain the following lower bound.
The actual pair $(x,y)$ lies in the disk $x^2+y^2\leq\delta^2$, so replacing it by the minimum over this disk can only decrease the right-hand side:
\begin{align*}
\frac{2\langle\bm{v}_a,\bm{w}_{ab}\rangle+\|\bm{w}_{ab}\|^2}
{2\|\bm{w}_{ab}+\bm{v}_a-\bm{v}_b\|}
&\geq
\min_{x^2+y^2\leq\delta^2}
\frac{2x\|\bm{w}_{ab}\|+\|\bm{w}_{ab}\|^2}
{2\big(\delta+\sqrt{(x+\|\bm{w}_{ab}\|)^2+y^2}\big)}\\
&\geq
\min_{|x|\leq\delta}
\frac{2x\|\bm{w}_{ab}\|+\|\bm{w}_{ab}\|^2}
{2\big(\delta+\sqrt{(x+\|\bm{w}_{ab}\|)^2+\delta^2-x^2}\big)}\\
&=
\min_{|x|\leq\delta}
\frac{\sqrt{\|\bm{w}_{ab}\|^2+2x\|\bm{w}_{ab}\|+\delta^2}-\delta}{2}\\
&\geq\frac{\Delta}{2}-\delta
>1.
\end{align*}
This contradicts \eqref{eq: cycle necessary inequality} and completes the proof.
\end{proof}

\subsection{Proofs of Theorems~\ref{thm.k=2 misclustering rate} and~\ref{thm:misclutering ratio}}
\label{sec: probabilistic conditions}

\begin{proof}[Proof of Theorem~\ref{thm.k=2 misclustering rate}]
For part~(a), the case $\Delta\cos\theta>2$ follows from Lemma~\ref{lem.delta cos theta}.
Suppose $\Delta\cos\theta=2$ and one-step recovery fails.
The proof of Lemma~\ref{lem.delta cos theta} then produces a misclustered pair $i,j$ for which
\[
2
\leq
\bigg\langle\bm{g}_i-\bm{g}_j,
\frac{\bmu_1^0-\bmu_2^0}{\|\bmu_1^0-\bmu_2^0\|}\bigg\rangle
\leq\|\bm{g}_i\|+\|\bm{g}_j\|
\leq2.
\]
Equality throughout forces
\[
\bm{g}_i=
\frac{\bmu_1^0-\bmu_2^0}{\|\bmu_1^0-\bmu_2^0\|},
\qquad
\bm{g}_j=-
\frac{\bmu_1^0-\bmu_2^0}{\|\bmu_1^0-\bmu_2^0\|}.
\]
For $d\geq2$, rotational invariance assigns probability zero to either prescribed direction.
There are finitely many pairs $i,j$, and so one-step recovery holds almost surely at the boundary as well.

For (b), we first follow the proof of Lemma~\ref{lem.delta cos theta}: 
We re-index $\bmu^0_1$ and $\bmu^0_2$ as necessary so that
\[
\bigg\langle \frac{\bmu^0_1-\bmu^0_2}{\|\bmu^0_1-\bmu^0_2\|},\frac{\hmu_1-\hmu_2}{\|\hmu_1-\hmu_2\|}\bigg\rangle
\geq0,
\]
and then note that any misclustered indices $i\in C^1_1\setminus C^\natural_1$ and $j\in C^1_2\setminus C^\natural_2$ necessarily satisfy
\[
\bigg\langle \bm{g}_i-\bm{g}_j,\frac{\bmu^0_1-\bmu^0_2}{\|\bmu^0_1-\bmu^0_2\|}\bigg\rangle
\geq \Delta\cos\theta.
\]
For each $i\in C^\natural_2$ and $j\in C^\natural_1$, let $B_{ij}$ indicate the event that
$\langle \bm{g}_i-\bm{g}_j,\frac{\bmu^0_1-\bmu^0_2}{\|\bmu^0_1-\bmu^0_2\|}\rangle
\geq \Delta\cos\theta$.
Since $|C^1_1\setminus C^\natural_1|=|C^1_2\setminus C^\natural_2|$, then the one-step misclustering rate $R_1$ satisfies
\[
R_1^2
=\bigg(\frac{|C^1_1\setminus C^\natural_1|+|C^1_2\setminus C^\natural_2|}{n}\bigg)^2
=\frac{|C^1_1\setminus C^\natural_1|\cdot |C^1_2\setminus C^\natural_2|}{(n/2)^2}
\leq\frac{1}{(n/2)^2}\sum_{i\in C^\natural_2}\sum_{j\in C^\natural_1}B_{ij}.
\]
Now select an ensemble of bijections $\pi_1,\ldots,\pi_{n/2}\colon C^\natural_2\to C^\natural_1$ such that for every $i\in C^\natural_2$ and $j\in C^\natural_1$, there is a unique $k\in[n/2]$ such that $\pi_k(i)=j$.
(For example, after indexing $C^\natural_2=\{a_i:i\in[n/2]\}$ and $C^\natural_1=\{b_i:i\in[n/2]\}$, one could define each $\pi_k$ by $a_i\mapsto b_{i+k}$, where $i+k$ is interpreted modulo $n/2$; one can think of this as a $1$-factorization of the complete bipartite graph $K_{n/2,n/2}$.)
Then
\[
R_1^2
\leq\frac{1}{(n/2)^2}\sum_{i\in C^\natural_2}\sum_{j\in C^\natural_1}B_{ij}
=\frac{1}{n/2}\sum_{k\in [n/2]}\frac{1}{n/2}\sum_{i\in C^\natural_2}B_{i,\pi_k(i)}.
\]
The terms of this sum are Bernoulli random variables with some common success probability~$p$, and so our upper bound on $R_1^2$ has expectation $p$.
We convert this to a high-probability bound by leveraging concentration of measure.
For fixed $s$, the variables $B_{i,\pi_s(i)}$ are independent because the bijection uses every noise vector at most once, and so Hoeffding's inequality (Theorem 2.2.6 in \cite{vershynin2018high}) gives
\[
\mathbb{P}\bigg\{ \frac{1}{n/2} \sum_{i\in C_2^\natural} B_{i,\pi_k(i)} \ge p + t \bigg\} 
\le  \exp(-nt^2).
\]
Next, the union bound delivers an estimate of the entire sum:
\begin{align*}
\mathbb{P}\bigg\{
\frac{1}{(n/2)^2} \sum_{i\in C_2^\natural} \sum_{j\in C_1^\natural} B_{ij} \ge p + t \bigg\} 
& \le \mathbb{P}\bigg\{ \frac{1}{n/2}\sum_{k\in [n/2]}\frac{1}{n/2} \sum_{i\in C_2^\natural} B_{i,\pi_k(i)} \ge p + t \bigg\}
\\
& \le \sum_{k\in[n/2]} \mathbb{P}\bigg\{ \frac{1}{n/2} \sum_{i\in C_2^\natural} B_{i,\pi_k(i)} \ge p + t \bigg\}
\le \frac{n}{2}\cdot \exp(-nt^2).
\end{align*}
We will take $t:=\sqrt{\frac{1}{n}\log\left(\frac{n}{2\varepsilon}\right)}$ so that this failure probability is at most $\varepsilon$.
It remains to estimate $p$.
For any $\lambda>0$, 
\begin{align*}
    p &= \mathbb{P}\left\{ \bigg\langle 
    \bm{g}_i - \bm{g}_j,\frac{\bmu^0_1 - \bmu^0_2}{\| \bmu^0_1 - \bmu^0_2 \|} \bigg\rangle \ge \Delta\cos\theta
    \right\} =  \mathbb{P}\left\{
    \exp\left(
    \lambda
    \bigg\langle 
    \bm{g}_i - \bm{g}_j,\frac{\bmu^0_1 - \bmu^0_2}{\| \bmu^0_1 - \bmu^0_2 \|} \bigg\rangle
    \right)
    \ge e^{\lambda(\Delta\cos\theta)}
    \right\}. 
\end{align*}
By Markov's inequality, 
\begin{align*}
    p&\le e^{-\lambda(\Delta\cos\theta)}\cdot \mathbb{E}
        \exp\left(
        \lambda
        \bigg\langle 
        \bm{g}_i - \bm{g}_j,\frac{\bmu^0_1 - \bmu^0_2}{\| \bmu^0_1 - \bmu^0_2 \|} \bigg\rangle
    \right).
\end{align*}
Since $\bm{g}_i$ and $\bm{g}_j$ are independent, the expectation term above becomes
\begin{align*}
    \mathbb{E}
        \exp\left(
        \lambda
        \bigg\langle 
        \bm{g}_i - \bm{g}_j,\frac{\bmu^0_1 - \bmu^0_2}{\| \bmu^0_1 - \bmu^0_2 \|} \bigg\rangle
    \right)
     & = \mathbb{E}
        \exp\left(
        \lambda
        \bigg\langle 
        \bm{g}_i, \frac{\bmu^0_1 - \bmu^0_2}{\| \bmu^0_1 - \bmu^0_2 \|} \bigg\rangle
    \right)
    \cdot 
    \mathbb{E} 
        \exp\left(
        \lambda
        \bigg\langle 
        -\bm{g}_j, \frac{\bmu^0_1 - \bmu^0_2}{\| \bmu^0_1 - \bmu^0_2 \|} \bigg\rangle
    \right).
\end{align*}
Moreover, rotational invariance implies that $-\bm{g}_j$ has the same distribution as $\bm{g}_j$. 
Since $\bm{g}_i$ and $\bm{g}_j$ are i.i.d., the preceding expression further reduces to 
\begin{align*}
    \mathbb{E}
        \exp
        \lambda
        \bigg\langle 
        \bm{g}_i - \bm{g}_j,\frac{\bmu^0_1 - \bmu^0_2}{\| \bmu^0_1 - \bmu^0_2 \|} \bigg\rangle
     & =
    \left(
    \mathbb{E} 
        \exp\left(
        \lambda
        \bigg\langle 
        \bm{g}, \frac{\bmu^0_1 - \bmu^0_2}{\| \bmu^0_1 - \bmu^0_2 \|} \bigg\rangle
    \right)
    \right)^2.
\end{align*}
Using the polar decomposition $\bm{g} = \|\bm{g}\|\cdot \frac{\bm{g}}{\|\bm{g}\|}$ and the fact that the radial component $\|\bm{g}\|$ is independent of the angular component $\frac{\bm{g}}{\|\bm{g}\|}$, we obtain 
\begin{align*}
    \mathbb{E}
        \exp\left(
        \lambda
        \bigg\langle 
        \bm{g}, \frac{\bmu^0_1 - \bmu^0_2}{\| \bmu^0_1 - \bmu^0_2 \|} \bigg\rangle
    \right)& = \mathbb{E}
        \exp\left(
        \lambda \|\bm{g}\|
        \bigg\langle 
       \frac{\bm{g}}{\|\bm{g}\|}, \frac{\bmu^0_1 - \bmu^0_2}{\| \bmu^0_1 - \bmu^0_2 \|} \bigg\rangle
    \right) \\
    & = \mathbb{E}_{\|\bm{g}\|} 
    \left[ 
    \mathbb{E}_{\bm{g}/\|\bm{g}\|}
    \left[
    \exp\left(
        \lambda \|\bm{g}\|
        \bigg\langle 
       \frac{\bm{g}}{\|\bm{g}\|}, \frac{\bmu^0_1 - \bmu^0_2}{\| \bmu^0_1 - \bmu^0_2 \|} \bigg\rangle
    \right) \bigg| \|\bm{g}\|
    \right]
    \right] \\
    & \le 
    \mathbb{E}_{\|\bm{g}\|} 
    \left[ 
        \exp\left(
            \frac{(\lambda\|\bm{g}\|)^2 }{2(d-1)}
        \right)
    \right] 
    \\
    & \le \exp\left(
            \frac{\lambda^2 }{2(d-1)}
        \right)
\end{align*}
where the second equality follows from the tower property, 
the first inequality applies 
Proposition~10.3.1 in~\cite{bobkov2023concentration} conditionally on $\|\bm{g}\|$ combined with the fact that $\frac{\bm{g}}{\|\bm{g}\|}$ is uniformly distributed on $\mathbb{S}^{d-1}$, and the last inequality uses $\|\bm{g}\| \le 1$.
Substituting this back into the preceding estimate for $p$ yields 
\begin{align*}
    p \le e^{-\lambda(\Delta\cos\theta)} \left( \exp\left(
    \frac{\lambda^2}{2(d-1)}
    \right)
    \right)^2 = \exp\left(
        -\lambda(\Delta\cos\theta) + \frac{\lambda^2}{d-1}
    \right)\quad \text{for all }\lambda>0.
\end{align*}  Let $\phi(\lambda) := -\lambda(\Delta\cos\theta) + \frac{\lambda^2}{d-1}$. This is a convex quadratic function in $\lambda$ whose minimal value is attained at $\lambda_0 = (d-1)(\Delta\cos\theta)/2 $, provided that $\Delta\cos\theta>0$. Therefore, 
\[
p    \le \exp\left(
    \phi(\lambda_0)
 \right)
    = \exp\left(
    -\frac{d-1}{4}\cdot (\Delta\cos\theta)^2
 \right).
\qedhere
\]
\end{proof}

\begin{proof}[Proof of Theorem~\ref{thm:misclutering ratio}]
We borrow many ideas from the proof of Theorem~\ref{thm:informal_basin_of_attraction}.
First, we re-index $\bmu^0$ as necessary so that $\max_{j\in[k]}\|\bmu^0_j-\hmu_j\|<\Delta/2$.
Next, for each $j\in[k]$, an Eulerian-digraph argument (\`{a}~la the proof of Theorem~\ref{thm:informal_basin_of_attraction}) gives that for every $i\in C^1_j\setminus C^\natural_j$, there exists $p=p_j(i)\in C^1_a\cap C^\natural_b$ for some $a=a_j(i)$ and $b=b_j(i)$ in $[k]$ with $a\neq b$ such that
\begin{equation}
\label{eq.necessary inequality per miscluster}
\frac{\|\bmu^0_{a} - \hmu_{b}\|^2 - \|\bmu^0_{a} - \hmu_{a}\|^2 }{2\|\bmu^0_{a} - \bmu^0_{b}\|}
\leq\left\langle \bm{g}_{p} , \frac{\bmu^0_{a} - \bmu^0_{b}}{\|\bmu^0_{a} - \bmu^0_{b}\|}\right\rangle.
\end{equation}
Furthermore, this map $p_j\colon C^1_j\setminus C^\natural_j\to[n]$ is injective since the underlying cycle decomposition consists of disjoint cycles.
Given $a,b\in[k]$ with $a\neq b$ and $p\in C^\natural_b$, let $B_{a,b,p}$ indicate the event that \eqref{eq.necessary inequality per miscluster} holds.
Then we have
\[
|C^1_j\setminus C^\natural_j|
\leq \sum_{\substack{a,b\in[k]\\a\neq b}}\sum_{p\in C^\natural_b} B_{a,b,p}.
\]
(Notably, this holds for every $j\in[k]$.)
As such, the misclustering rate $R_1$ satisfies
\[
R_1
=\frac{1}{n}\sum_{j\in[k]}|C^1_j\setminus C^\natural_j|
\leq\frac{k}{n}\sum_{\substack{a,b\in[k]\\a\neq b}}\sum_{p\in C^\natural_b} B_{a,b,p}.
\]
Much like the proof of Theorem~\ref{thm.k=2 misclustering rate}, the inner sum consists of i.i.d.\ Bernoulli random variables, and so it can be estimated using Hoeffding's inequality, and then the outer sum can be estimated using the union bound.
One can verify that the resulting bound on $R_1$ exhibits the claimed behavior.
\end{proof}

\subsection{Proof of Theorem~\ref{thm:informal_choices of initialization} and Corollary~\ref{cor:diameter-initialization-stochastic}}
\label{sec:initialization}

\begin{proof}[Proof of Theorem~\ref{thm:informal_choices of initialization}]
For (a), we first show that there are points in $C^\natural_1$ and $C^\natural_2$ of distance at least $\Delta$:
\begin{align*}
\max_{\substack{i\in C^\natural_1\\j\in C^\natural_2}}\|\bm{x}_i-\bm{x}_j\|^2
&\geq\frac{1}{(n/2)^2}\sum_{i\in C^\natural_1}\sum_{j\in C^\natural_2}\|\bm{x}_i-\bm{x}_j\|^2\\
&=\frac{1}{(n/2)^2}\sum_{i\in C^\natural_1}\sum_{j\in C^\natural_2}\|(\bm{x}_i-\hmu_1)-(\bm{x}_j-\hmu_2)+(\hmu_1-\hmu_2)\|^2\\
&=\frac{1}{n/2}\sum_{i\in C^\natural_1}\|\bm{x}_i-\hmu_1\|^2+\frac{1}{n/2}\sum_{j\in C^\natural_2}\|\bm{x}_j-\hmu_2\|^2+\|\hmu_1-\hmu_2\|^2\\
&\geq\|\hmu_1-\hmu_2\|^2\\[1em]
&=\Delta^2.
\end{align*}
Since points within each planted cluster have distance at most $2<2\sqrt{2}\leq\Delta$, the diameter is only achieved by points from different planted clusters, and so the result follows from Theorem ~\ref{thm:informal_basin_of_attraction}.

For (b), let $t_i\in\{1,\ldots,n\}$ denote the $i$th random index drawn for each $i\in\{1,\ldots,m\}$.
We will use a coupon-collecting argument to ensure that, with high probability, each planted cluster is sampled by $\{\bm{x}_{t_i}\}_{i\in[m]}$.
Indeed, the probability that some planted cluster is not sampled is at most $k$ times the probability that the first cluster is not sampled, which in turn is
\[
k\cdot \bigg(\frac{n-(n/k)}{n}\bigg)\bigg(\frac{n-(n/k)-1}{n-1}\bigg)\cdots\bigg(\frac{n-(n/k)-(m-1)}{n-(m-1)}\bigg)
\leq k\cdot\bigg(1-\frac{1}{k}\bigg)^m
\leq k\cdot e^{-m/k}
\leq \frac{\varepsilon}{2}.
\]
Next, we recall that the data points were drawn according to
\[
\bm{x}_{t}
=\hmu_{\sigma(t)}+\bm{g}_{t},
\]
with $\bm{g}_1,\ldots,\bm{g}_n$ being independent with the same distribution as some random vector $\bm{g}$.
Our assumption on $\mathbb{E}\|\bm{g}\|^2$ then allows us to apply Markov's inequality:
\begin{align*}
\mathbb{P}\bigg\{\exists\, i\in[m],~\|\bm{g}_{t_i}\|\geq\frac{\Delta}{2}-1
\bigg\}
&=\mathbb{E}\,\mathbb{P}_{t_1,\ldots,t_m}\bigg\{\exists \,i\in[m],~\|\bm{g}_{t_i}\|\geq\frac{\Delta}{2}-1
\bigg\}\\
&\leq m\cdot \mathbb{P}\bigg\{\|\bm{g}\|\geq\frac{\Delta}{2}-1\bigg\}\\
&\leq m\cdot\frac{\mathbb{E}\|\bm{g}\|^2}{(\frac{\Delta}{2}-1)^2}\\
&\leq\frac{\varepsilon}{2}.
\end{align*}
Overall, with probability at least $1-\varepsilon$, it holds that every planted cluster is sampled, and furthermore, each sample is within $\frac{\Delta}{2}-1$ of the corresponding ball center.

Next, any two of these $m$ points are ``adjacent'' if their distance is at most $\min\{\Delta-2,2\}$.
This occurs precisely when the points sample the same planted cluster.
Indeed, if $t_i$ and $t_j$ sample the same planted cluster, then
\[
\|\bm{x}_{t_i}-\bm{x}_{t_j}\|
<\|\bm{g}_{t_i}\|+\|\bm{g}_{t_j}\|
\leq 2\min\{\tfrac{\Delta}{2}-1,1\}
=\min\{\Delta-2,2\}.
\]
(The first inequality above is strict with probability $1$.)
On the other hand, if $t_i$ and $t_j$ sample different planted clusters, then
\[
\|\bm{x}_{t_i}-\bm{x}_{t_j}\|
>\Delta-\big(\|\bm{g}_{t_i}\|+\|\bm{g}_{t_j}\|\big)
\geq\Delta-\min\{\Delta-2,2\}
=\max\{\Delta-2,2\}
\geq\min\{\Delta-2,2\}.
\]
(The first inequality above is strict with probability $1$.)
The result then follows from Theorem ~\ref{thm:informal_basin_of_attraction}.
\end{proof}

\begin{proof}[Proof of Corollary~\ref{cor:diameter-initialization-stochastic}]
Given $\Delta>2\sqrt{2}$, select $\varepsilon=\varepsilon(\Delta)>0$ so that
\[
\frac{\Delta-2\varepsilon}{1+\varepsilon}
=\frac{1}{2}\Big(\Delta+2\sqrt{2}\Big).
\]
We claim that with probability $1-e^{-\Omega_\Delta(n)}$, the scaled data points $\{\frac{1}{1+\varepsilon}\bm{x}_i\}_{i\in[n]}$ satisfy the balanced deterministic ball model with ground truth centers given by the scaled empirical centroids
\[
\widetilde\bmu_j
:=\frac{1}{1+\varepsilon}\cdot\widehat\bmu_j,
\qquad
\widehat\bmu_j
:=\frac{1}{n/2}\sum_{\sigma(i)=j}\bm{x}_i,
\qquad
j\in\{1,2\},
\]
which in turn satisfy
\[
\|\widetilde\bmu_1-\widetilde\bmu_2\|
\geq
\frac{\Delta-2\varepsilon}{1+\varepsilon}
=\frac{1}{2}\Big(\Delta+2\sqrt{2}\Big)
>2\sqrt{2}.
\]
Since the noise vectors $\bm{v}_i:=\frac{1}{1+\varepsilon}\bm{x}_i-\widetilde\bmu_{\sigma(i)}$ satisfy
\[
\sum_{\sigma(i)=j}\bm{v}_i
=\bm{0}
\qquad
\forall\,j\in\{1,2\},
\]
the result then follows from Theorem~\ref{thm:informal_choices of initialization}(a) and the fact that the diameter initialization and the BalLOT iterations are all equivariant under positive scaling.
It suffices to show that
\[
\|\bm{x}_i-\widehat\bmu_{\sigma(i)}\|\leq 1+\varepsilon
\quad \forall\,i\in[n]
\qquad\text{and}\qquad
\|\widehat\bmu_1-\widehat\bmu_2\|\geq\Delta-2\varepsilon
\]
with probability $1-e^{-\Omega_\Delta(n)}$.
To this end, consider the empirical average noise vectors
\[
\overline{\bm{g}}_j
:=\frac{1}{n/2}\sum_{\sigma(i)=j}\bm{g}_i,
\qquad
j\in\{1,2\}.
\]
For each $j\in\{1,2\}$, the sequence $\{\bm{g}_i\}_{\sigma(i)=j}$ consists of i.i.d.\ bounded random vectors, so the vector Hoeffding inequality (see Example~6.3 in~\cite{boucheron2013concentration}) and the union bound together give
\[
\max\big\{\|\overline{\bm{g}}_1\|,\|\overline{\bm{g}}_2\|\big\}
\leq\varepsilon
\qquad
\text{with probability}
\qquad
1-e^{-\Omega_\Delta(n)}.
\]
(Here, the dependence on $\Delta$ in the probability factors through $\varepsilon=\varepsilon(\Delta)$.)
In this event, we have
\[
\|\bm{x}_i-\widehat\bmu_{\sigma(i)}\|
=\|\bm{g}_i-\overline{\bm{g}}_{\sigma(i)}\|
\leq\|\bm{g}_i\|+\|\overline{\bm{g}}_{\sigma(i)}\|
\leq 1+\varepsilon
\qquad
\forall\,i\in[n]
\]
and
\[
\|\widehat\bmu_1-\widehat\bmu_2\|
=\|(\hmu_1+\overline{\bm{g}}_{1})-(\hmu_2+\overline{\bm{g}}_{2})\|
\geq\|\hmu_1-\hmu_2\|-\|\overline{\bm{g}}_{1}\|-\|\overline{\bm{g}}_{2}\|
\geq\Delta-2\varepsilon,
\]
as desired.
\end{proof}

\section{Numerical experiments}
\label{sec: numerical experiments}

\subsection{Algorithms and implementation details}

We compare BalLOT with several baseline algorithms.
First, we consider the balanced $k$-means clustering semidefinite program:

\medskip
\noindent
\textbf{Amini--Levina SDP.}
We solve \eqref{prob: Amini Levina SDP} using SDPNAL+ and round the solution to a balanced clustering using the algorithm \texttt{cluster} from \cite{fei2022hidden}.
The SDP is included only at problem sizes for which it is computationally feasible.

\medskip

All other methods we compare with are based on alternating minimization.
For these, we use $k$-means++ initialization \cite{arthur2007kmeanspp}, and then each iteration consists of three steps: (1) computing a cost matrix, (2) performing an assignment step, and (3) computing new centroids.
The first and third steps are common to all methods, while the assignment step is method-dependent, described below.


\medskip
\noindent
\textbf{Lloyd's algorithm.}
Every data point is sent to its nearest current centroid.
This method is fast but need not return balanced clusters.

\medskip
\noindent
\textbf{Replicated-centroid matching.}
We make $n/k$ copies of every current centroid and solve a minimum-cost assignment between the $n$ data points and the $n$ replicated centroids.
We implement this both with the Hungarian algorithm, as implemented by \cite{Cao2025Hungarian}, and with MATLAB's built-in \texttt{matchpairs} function.
For \texttt{matchpairs}, we use an unmatched-pair penalty large enough to force a complete matching in every reported run.
The two formulations have the same optimal objective value as the BalLOT assignment step, but they solve an $n\times n$ assignment problem instead of the $n\times k$ transportation problem.

\medskip
\noindent
\textbf{BalLOT.}
We solve the transportation linear program over $\bUnk$ using MATLAB's \texttt{linprog}, and then update the centroids by $\bmu=k\bm{X}\bmF$.
The transportation problem has $nk$ variables and $n+k$ marginal constraints.

\medskip
\noindent
\textbf{E-BalLOT.}
We approximate the transportation problem using entropic optimal transport and Sinkhorn matrix scaling \cite{cuturi2013sinkhorn}.
Specifically, we approximately solve the following regularized problem by Sinkhorn iterations:
\[
\begin{aligned}
&\text{minimize} &&\quad
\sum_{i\in[n]}\sum_{j\in[k]}C_{ij}F_{ij}
+\lambda\sum_{i\in[n]}\sum_{j\in[k]}F_{ij}(\log F_{ij}-1)\\
&\text{subject to} &&\quad \bmF\in\bUnk.
\end{aligned}
\]
To define the Sinkhorn iterations, put
\[
\bm{a}:=\frac{1}{n}\bm{1}_n,
\qquad
\bm{b}:=\frac{1}{k}\bm{1}_k,
\qquad
(\bm{K}_\lambda)_{ij}:=\exp(-C_{ij}/\lambda).
\]
Starting from a positive vector $\bm{v}^0$, alternately rescale the rows and columns by
\[
\bm{u}^{s+1}:=\bm{a}\oslash(\bm{K}_\lambda\bm{v}^s),
\qquad
\bm{v}^{s+1}:=\bm{b}\oslash(\bm{K}_\lambda^T\bm{u}^{s+1}),
\]
where $\oslash$ denotes entrywise division, and form
\[
\bmF^{s+1}:=\diag(\bm{u}^{s+1})\bm{K}_\lambda\diag(\bm{v}^{s+1}).
\]
These are the Sinkhorn matrix-scaling iterations \cite{cuturi2013sinkhorn,Peyr2018ComputationalOT}.
We take $\lambda=0.05$ and stop once
\[
\bigg\|\bmF\bm{1}_k-\frac{1}{n}\bm{1}_n\bigg\|_1
+
\bigg\|\bmF^T\bm{1}_n-\frac{1}{k}\bm{1}_k\bigg\|_1
<0.01.
\]
We subsequently round the resulting approximate coupling back to $\bUnk$ using Algorithm 2 in \cite{altschuler2017near}.
We extract a clustering by assigning each row index to a row maximizer.
With suitable parameter choices, entropic optimal transport can produce an $\varepsilon$-approximate solution to the unregularized problem in
$O((\|\bm{C}\|_\infty/\varepsilon)^2kn\log n)$ operations \cite{dvurechensky2018computational}.
While our theory only concerns exact BalLOT, this approximate implementation is similar in spirit, and has previously appeared in~\cite{CuturiDoucet2014WassersteinBarycenters}.


\medskip
\noindent
\textbf{Oracle BalLOT.}
In planted-model experiments, we include a one-step oracle that solves one exact BalLOT assignment problem initialized at the planted centers.
(The naive oracle of nearest-center assignment would not be appropriate here, since it can produce clusters of unequal size.)

\subsection{Recovery and runtime under the stochastic ball model}

\begin{experiment}
\label{exp.exact recovery rate vs Delta}
Consider the balanced stochastic ball model with $n=100$ data points and $k=2$ planted clusters of equal size, where $\bm{g}$ is uniformly distributed on the unit ball.
For $d=2$, let $\Delta$ range from $1.50$ to $2.30$ in increments of $0.05$; for $d=10$, let $\Delta$ range from $1.00$ to $1.80$ in increments of $0.05$.
For each value of $\Delta$ and each dimension, run $500$ trials and record the fraction of trials in which each algorithm exactly recovers the planted clustering.
The results appear in Figure~\ref{fig:recovery_fraction}.
The Hungarian, Matchpairs, and exact BalLOT implementations coincide in these trials, as expected from the equivalence of their exact balanced assignment steps.
The higher-dimensional experiment exhibits exact recovery at smaller separations than the two-dimensional experiment.
\end{experiment}

\begin{figure}[t]
\centering
\includegraphics[width=0.47\linewidth]{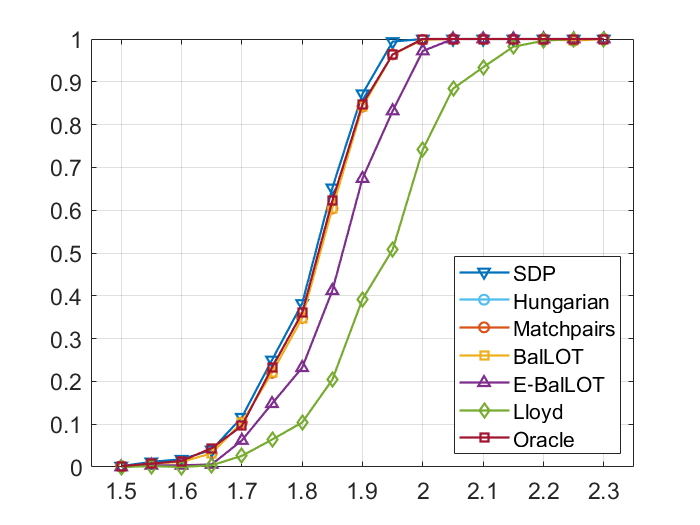}
\qquad
\includegraphics[width=0.47\linewidth]{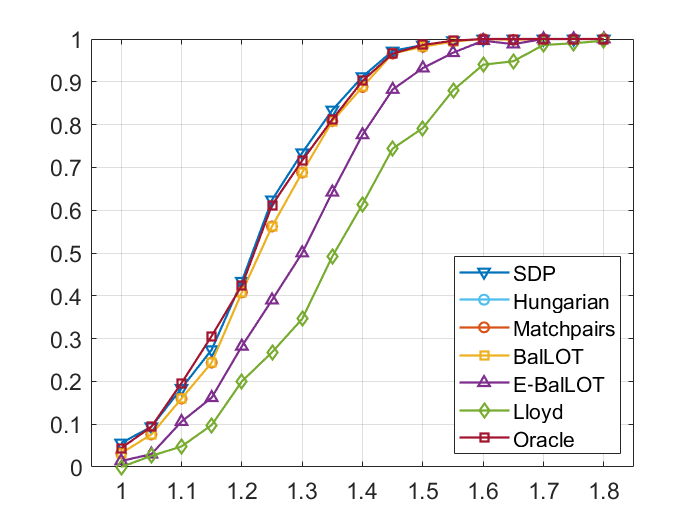}
\caption{Exact recovery under the balanced stochastic ball model with $n=100$ and $k=2$.
For each value of the separation parameter $\Delta$, we run $500$ trials and plot the fraction in which the planted clustering is exactly recovered.
The oracle method runs one exact BalLOT assignment step initialized at the planted centers.
\textbf{(left)} $d=2$.
\textbf{(right)} $d=10$.
See Experiment~\ref{exp.exact recovery rate vs Delta} for details.}
\label{fig:recovery_fraction}
\end{figure}

\begin{experiment}
\label{exper.runtime comparison}
Again consider the balanced stochastic ball model with $d=2$, $k=2$, and $\Delta=3$.
For each $n\in\{2^2,2^3,\ldots,2^{27}\}$, run $20$ trials and record the runtime of each algorithm.
Figure~\ref{fig:runtime} displays the median runtime.
For each algorithm, we stop increasing $n$ once the median runtime exceeds $10$ seconds, except that the Hungarian and Matchpairs implementations were stopped earlier after their runtimes became extraordinarily long.
In these experiments, BalLOT, E-BalLOT, and Lloyd's algorithm exhibit near-linear empirical growth in $n$, while the SDP and replicated-centroid matching methods grow substantially faster.
\end{experiment}

\begin{figure}[p]
\centering
\begin{tikzpicture}
\node[anchor=west] at (0,0) {\includegraphics[width=0.7\linewidth]{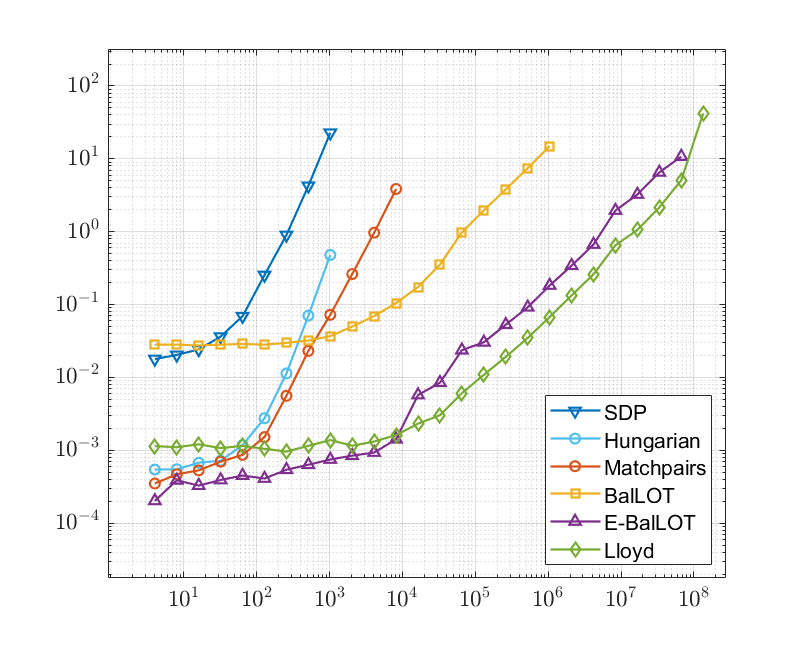}};
\node[rotate=90, anchor=center] at (0.5,0) {\small{median runtime in seconds}};
\node at (10.3,-3.75) {$n$};
\end{tikzpicture}
\caption{Median runtime in Experiment~\ref{exper.runtime comparison}.
For each $n\in\{2^2,2^3,\ldots,2^{27}\}$, we draw $n$ data points in $\mathbb{R}^2$ from a balanced stochastic ball model with two clusters.
BalLOT, E-BalLOT, and Lloyd's algorithm exhibit near-linear empirical runtimes in this experiment, while the other methods are super-linear.}
\label{fig:runtime}
\end{figure}

\subsection{Gaussian mixture experiments}

\begin{experiment}
\label{exper.gmm}
Fix $n=2000$, $k=5$, and $d\in\{2,10\}$.
Draw planted means $\hmu_1,\ldots,\hmu_k\in\mathbb{R}^d$ with i.i.d.\ $N(0,4)$ coordinates, draw displacements $\bm{g}_1,\ldots,\bm{g}_n$ with i.i.d.\ $N(0,1)$ coordinates, fix a balanced assignment $\sigma\colon[n]\to[k]$, and put $\bm{x}_i=\hmu_{\sigma(i)}+\bm{g}_i$.
Generate $50$ data sets for each $d$, and for each realization compute $10$ independent $k$-means++ initializations.
Run Lloyd's algorithm, BalLOT, and E-BalLOT from the same initializations, and also run oracle BalLOT from the planted means.
For each run, compute the $2$-Wasserstein distance between the uniform empirical measure on the estimated centroids and the uniform empirical measure on the planted means.
Figure~\ref{fig:gmm_std2_d2_d10_with_oracle} summarizes the resulting $50\times10$ trials.
The first boxplot in each panel records the error of the $k$-means++ initialization itself; the remaining boxplots record the output of the indicated algorithms.
The Amini--Levina SDP is omitted because the SDP variable is $2000\times2000$, and preliminary runs with SDPNAL+ exceeded our computational budget.
This is consistent with the runtime behavior in Experiment~\ref{exper.runtime comparison}.
\end{experiment}

\begin{figure}[p]
\centering
\includegraphics[width=0.47\linewidth]{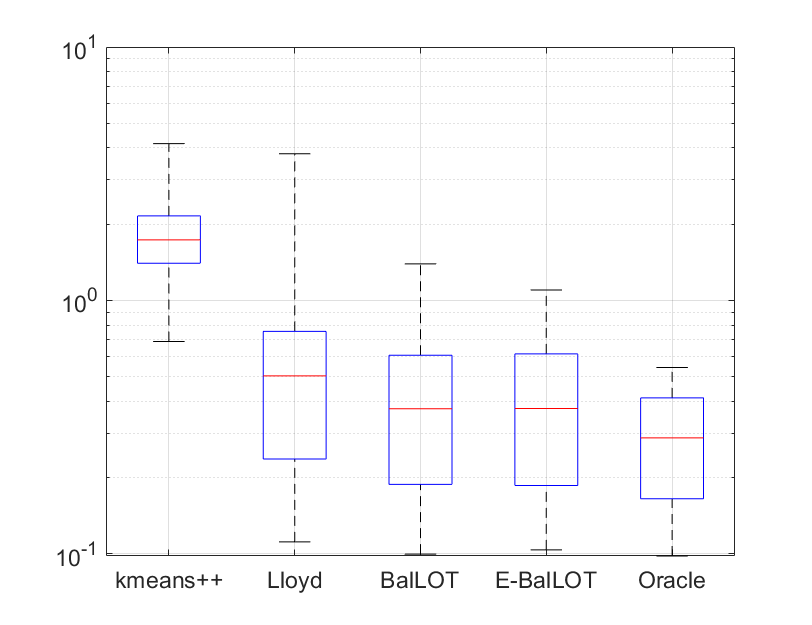}
\qquad
\includegraphics[width=0.47\linewidth]{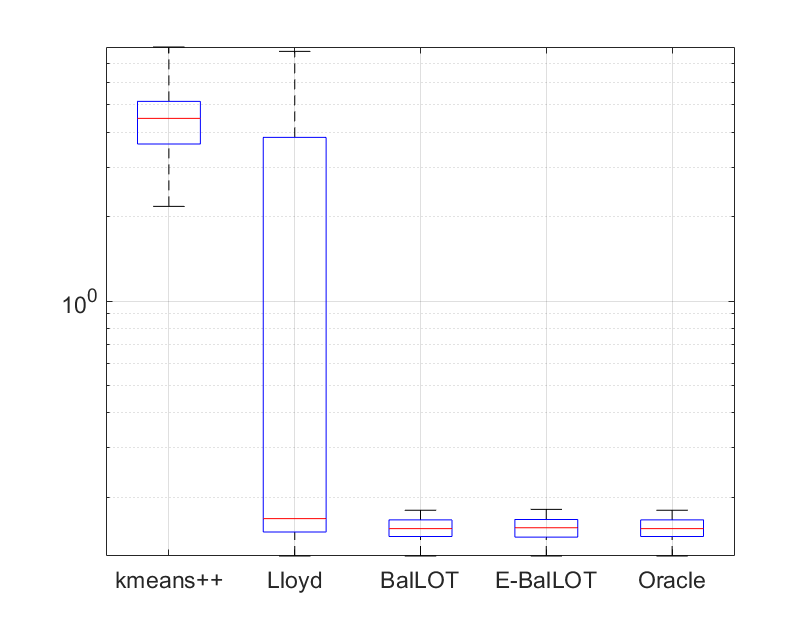}
\caption{Estimating centers in a balanced Gaussian mixture model when the planted means have i.i.d.\ $N(0,4)$ coordinates.
The first boxplot records the distance from the $k$-means++ initialization to the planted means.
The remaining boxplots record the distance after running the indicated algorithms, with oracle BalLOT initialized at the planted means.
\textbf{(left)} $d=2$.
\textbf{(right)} $d=10$.
See Experiment~\ref{exper.gmm}.}
\label{fig:gmm_std2_d2_d10_with_oracle}
\end{figure}


\subsection{The deterministic basin of attraction}
\label{sec: deterministic basin experiments}

\begin{experiment}
\label{exper.k=2}
For each $\Delta\in\{2.1,2.2,\ldots,3.0\}$ and each $\delta\in\{0.1,0.2,\ldots,1.0\}$, conduct $200$ trials of the following experiment.
Draw $n=100$ points from the balanced stochastic ball model with $k=2$ and $\bm{g}$ uniformly distributed on the unit circle.
Take $\bmu_1^0$ and $\bmu_2^0$ to be random $\delta$-perturbations of $\hmu_1$ and $\hmu_2$, respectively, and run one step of BalLOT.
Figure~\ref{fig:heat_map_2_means}(left) displays the fraction of trials that exactly recover the planted clustering.
The threshold from Theorem~\ref{thm:informal_basin_of_attraction} is overlaid in red.
Every pixel below this curve is white, as guaranteed, while the nonwhite pixels immediately above it indicate that the threshold seems sharp experimentally.
\end{experiment}

\begin{experiment}
\label{exper.k=3}
Repeat Experiment~\ref{exper.k=2} with $k=3$, $n=300$, and planted centers at the vertices of an equilateral triangle with side length $\Delta$.
Figure~\ref{fig:heat_map_2_means}(right) displays the fraction of trials that exactly recover the planted clustering.
The $k>2$ threshold from Theorem~\ref{thm:informal_basin_of_attraction} is shown in red, while the $k=2$ threshold is shown in blue.
Many pixels above the red curve remain white, illustrating that the $k>2$ sufficient condition is not sharp.
The gray pixels below the blue curve show that the $k=2$ threshold does not extend to this setting.
\end{experiment}

\begin{figure}[p]
\centering
\begin{tikzpicture}
\node[anchor=west] at (0,0) {\includegraphics[trim={3cm 0 3cm 0},clip,width=2.5in]{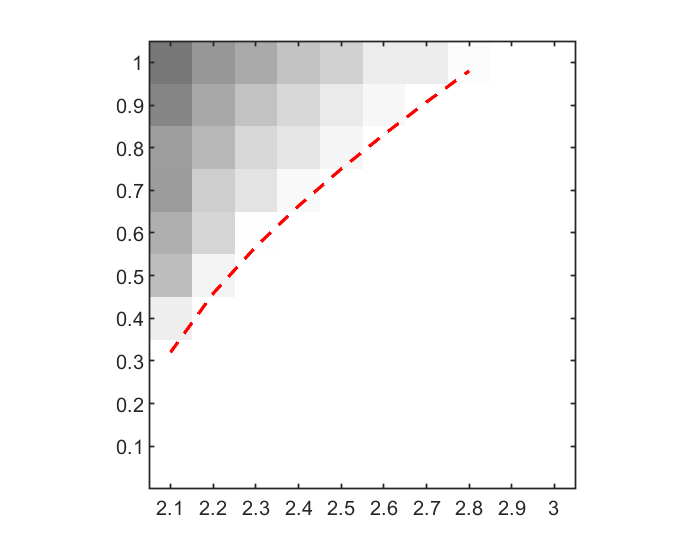}};
\node[rotate=90,anchor=center] at (-0.3,0) {\small{initial distance from true means}};
\node at (6.6,-3.15) {$\Delta$};
\node at (5.2,2) {\textcolor{red}{$\sqrt{(\frac{\Delta}{2})^2-1}$}};
\end{tikzpicture}
\qquad
\begin{tikzpicture}
\node[anchor=west] at (0,0) {\includegraphics[trim={3cm 0 3cm 0},clip,width=2.5in]{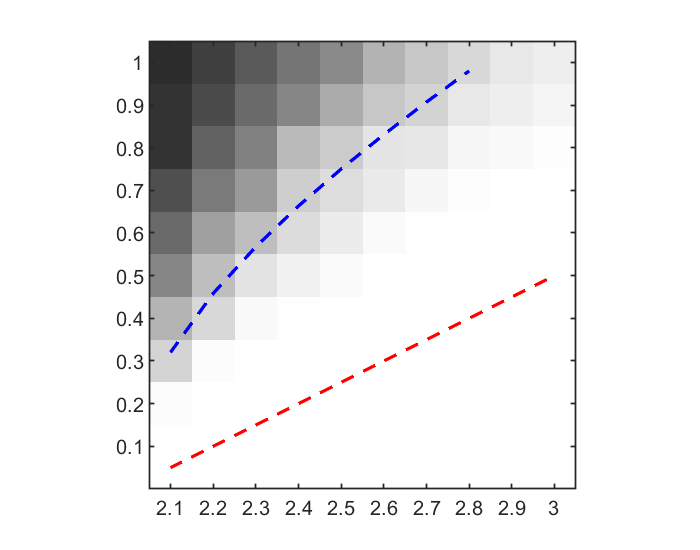}};
\node[rotate=90,anchor=center] at (-0.3,0) {\small{initial distance from true means}};
\node at (6.6,-3.15) {$\Delta$};
\node at (5.2,2) {\textcolor{blue}{$\sqrt{(\frac{\Delta}{2})^2-1}$}};
\node at (5.4,-1) {\textcolor{red}{$\frac{\Delta}{2}-1$}};
\end{tikzpicture}
\caption{Probability of exact one-step recovery in Experiments~\ref{exper.k=2} and~\ref{exper.k=3}.
White denotes probability $1$ and black denotes probability $0$.
\textbf{(left)} $k=2$.
\textbf{(right)} $k=3$.
The sufficient thresholds from Theorem~\ref{thm:informal_basin_of_attraction} are overlaid.}
\label{fig:heat_map_2_means}
\end{figure}

\subsection{First-step and final recovery}

\begin{experiment}
\label{exper.probabilistic log decaying rate}
Fix $d=25$, $n=2000$, and $k=2$.
For each $\eta\in\{0.075,0.150,\ldots,1.5\}$ and $\Delta\in\{2.1,2.2,\ldots,4.0\}$, put $\theta:=\arccos(\eta/\Delta)$ so that $\eta=\Delta\cos\theta$, and run $10$ trials.
Draw data from the corresponding balanced stochastic ball model with $\|\bm{g}_i\|=1$ almost surely, run BalLOT from an initialization with the prescribed angle, and record both the first-step misclustering rate and the final misclustering rate after termination.
Figure~\ref{fig:misclustering_rate_full_and_first_step} displays the results.
The red dashed curve is the $n\to\infty$ upper bound $\exp(-(d-1)(\Delta\cos\theta)^2/8)$ from Theorem~\ref{thm.k=2 misclustering rate}(b).
It captures the empirical decay rate, while every full BalLOT run in this experiment exactly recovers the planted clustering.
\end{experiment}

\begin{figure}[p]
\centering
\includegraphics[width=0.47\linewidth]{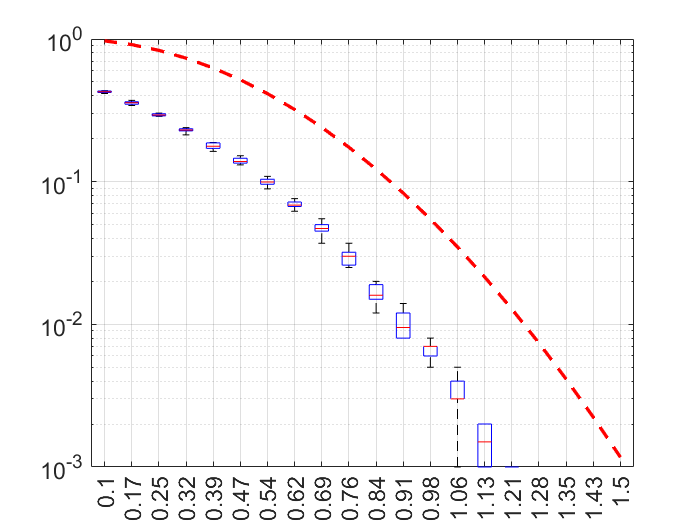}
\qquad
\includegraphics[width=0.47\linewidth]{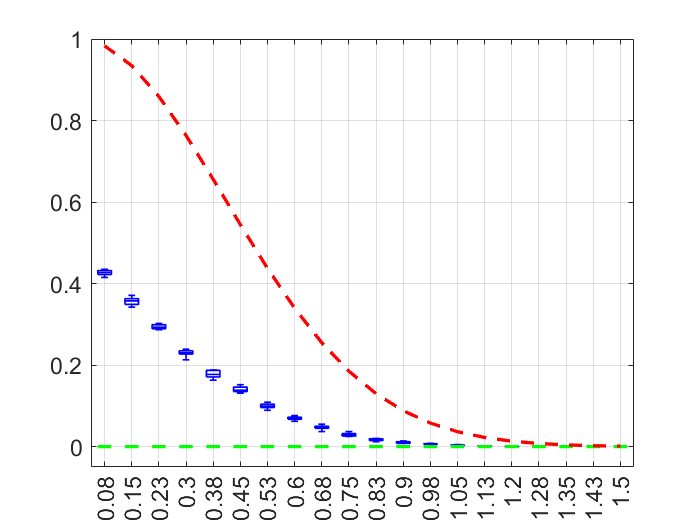}
\caption{First-step and final misclustering rates for BalLOT under a balanced stochastic ball model with $d=25$, $n=2000$, and $k=2$.
\textbf{(left)} First-step misclustering rates on a logarithmic scale, together with the asymptotic upper bound from Theorem~\ref{thm.k=2 misclustering rate}(b).
\textbf{(right)} First-step and final misclustering rates.
The dashed curve is an upper bound, not an exact limit.
See Experiment~\ref{exper.probabilistic log decaying rate}.}
\label{fig:misclustering_rate_full_and_first_step}
\end{figure}

\begin{experiment}
\label{exper.two-step recovery phase transition}
Fix $d=2$, $n=10000$, and $k=2$.
For each $\Delta\in\{2.1,2.2,\ldots,4.0\}$ and $\cos\theta\in\{0.05,0.10,\ldots,1.00\}$, run $20$ trials.
Draw data from the corresponding balanced stochastic ball model with $\|\bm{g}_i\|=1$ almost surely, run one step of BalLOT with an appropriate initialization, and record the resulting misclustering rate.
Figure~\ref{fig: SNR_0 and R_1} displays the results, with the threshold $\Delta\cos\theta=2$ overlaid in red.
The value of $\Delta\cos\theta$ is a good predictor of the first-step misclustering rate, and the white region above the curve agrees with Theorem~\ref{thm.k=2 misclustering rate}(a).
\end{experiment}

\begin{figure}[p]
\centering
\begin{tikzpicture}
\node[anchor=west] at (0,0) {\includegraphics[width=2.5in]{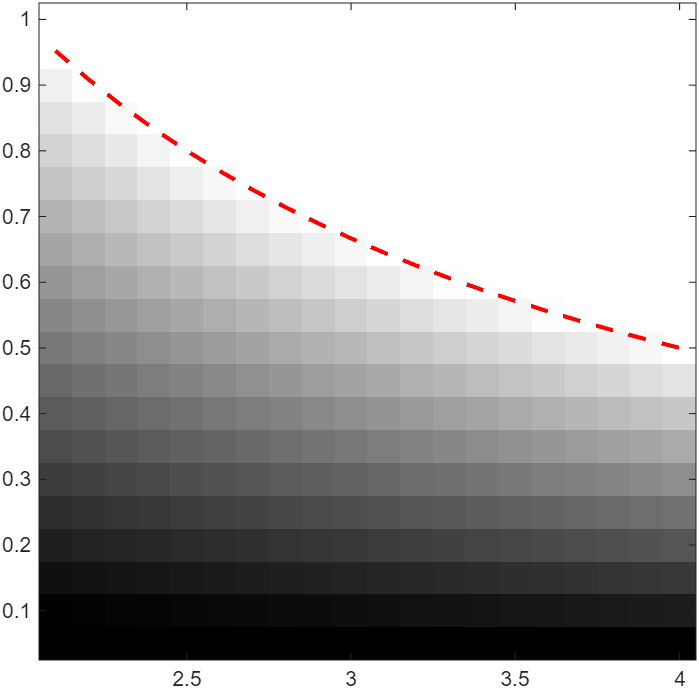}};
\node[rotate=90,anchor=center] at (-0.2,0) {$\cos\theta$};
\node at (6.7,-3.30) {$\Delta$};
\node at (4,1.5) {\textcolor{red}{$\Delta\cos\theta=2$}};
\end{tikzpicture}
\hfill
\begin{tikzpicture}
\node[anchor=west] at (0,0) {\includegraphics[width=2.5in]{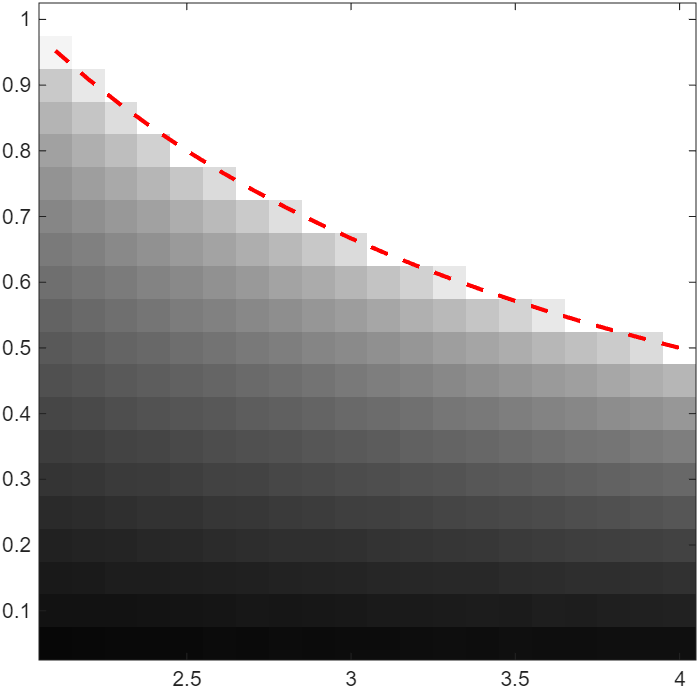}};
\node[rotate=90,anchor=center] at (-0.2,0) {$\cos\theta$};
\node at (6.7,-3.30) {$\Delta$};
\node at (4,1.5) {\textcolor{red}{$\Delta\cos\theta=2$}};
\end{tikzpicture}
\caption{The phase transition in Experiment~\ref{exper.two-step recovery phase transition}.
\textbf{(left)} Heat map of $\Delta\cos\theta$, with white denoting $\Delta\cos\theta\geq2$.
\textbf{(right)} Heat map of the average first-step misclustering rate, with white denoting one-step recovery.
The red dashed curve marks $\Delta\cos\theta=2$.}
\label{fig: SNR_0 and R_1}
\end{figure}

\subsection{Coupon-collecting initialization}

\begin{experiment}
\label{exper.proto-means}
Fix $d=2$, $n=1200$, and $k=3$.
Let $\hmu_1,\hmu_2,\hmu_3$ form the vertices of an equilateral triangle with side length $\Delta=3$.
Put $\varepsilon=0.4$, and run $1500$ trials.
Draw data from the corresponding stochastic ball model with $\|\bm{g}_i\|\sim(\operatorname{Unif}[0,1])^\alpha$, where $\alpha$ is chosen so that the second-moment condition in Theorem~\ref{thm:informal_choices of initialization}(b) holds with equality.
Apply the coupon-collecting initialization and record the smallest $r$ for which every one of the $m=9$ sampled points lies in a radius-$r$ ball centered at a planted center.
Figure~\ref{fig: examples of initialization schemes} displays the results.
In $93.9\%$ of the trials, $r\leq1/2=\Delta/2-1$, in which case Theorem~\ref{thm:informal_basin_of_attraction} guarantees one-step recovery.
\end{experiment}

\begin{figure}[p]
\centering
\begin{tikzpicture}
\node[anchor=west]{\includegraphics[width=0.45\linewidth]{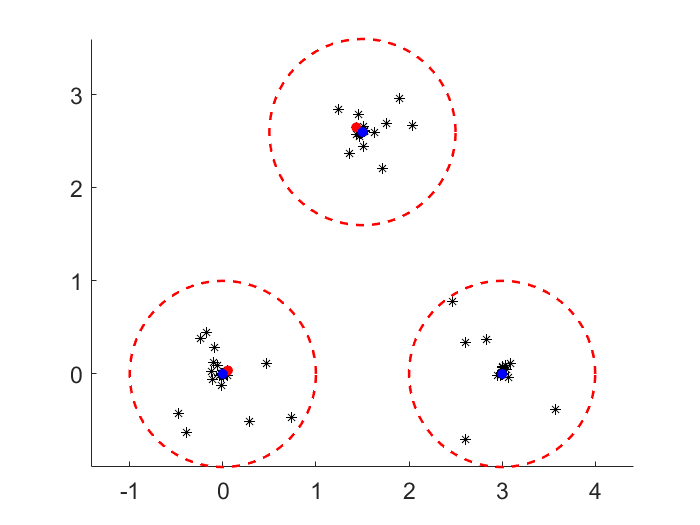}};
\node[anchor=center] at (0.8,2.4) {$x_2$};
\node at (6.8,-2.3) {$x_1$};
\end{tikzpicture}
\qquad
\begin{tikzpicture}
\node[anchor=west]{\includegraphics[width=0.45\linewidth]{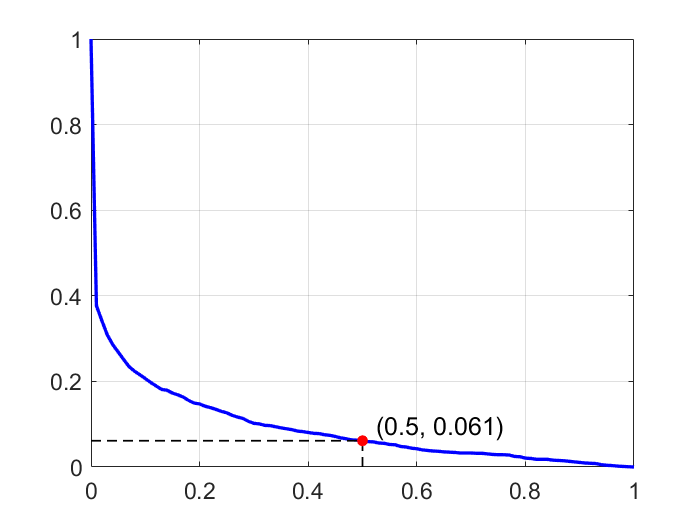}};
\node[rotate=90,anchor=center] at (0.2,0) {\small{ratio of trials with $r>t$}};
\node at (6.8,-2.30) {$t$};
\end{tikzpicture}
\caption{$m$-coupon-collecting initialization in Experiment~\ref{exper.proto-means}.
\textbf{(left)} One data realization and its sampled points for the initialization.
\textbf{(right)} For each trial, we record the smallest $r$ for which all of the sampled points lie within distance $r$ of planted centers.
The point $(0.5,0.061)$ indicates that $93.9\%$ of the trials satisfy $r\leq1/2$, the sufficient condition from Theorem~\ref{thm:informal_basin_of_attraction}.}
\label{fig: examples of initialization schemes}
\end{figure}

\section{Discussion}
\label{sec.discussion}

In this paper, we introduced BalLOT and proved several performance guarantees.
This section presents several opportunities for follow-on work.

\medskip
\noindent
\textbf{Convergence analysis.}
While we identified sufficient conditions for one-step recovery under the stochastic ball model, we empirically observe that full BalLOT frequently converges to the planted clusters; see Experiments~\ref{exp.exact recovery rate vs Delta} and~\ref{exper.probabilistic log decaying rate}.
A multi-step finite-sample convergence analysis would therefore be interesting.

\medskip
\noindent
\textbf{Unbalanced clustering.}
One might be interested in generalizations in which unequal target cluster sizes are specified.
This includes the case where one wants a clustering that is as balanced as possible but $k$ does not divide $n$.
Cyclical monotonicity is available in this more general setting \cite{santambrogio2015optimal,DePascal2023cyclical}, but the aggregated coupling in the proof of Theorem~\ref{thm: informal_average_landscape} is no longer doubly stochastic.
Does the average landscape have any spurious local minimizers in this setting?

\medskip
\noindent
\textbf{Other mixture models.}
Our deterministic and probabilistic finite-sample guarantees focus on stochastic ball models, while Experiment~\ref{exper.gmm} indicates that BalLOT also performs well for Gaussian mixture models.
Can one derive comparable finite-sample guarantees for Gaussian or sub-Gaussian mixtures?

\medskip
\noindent
\textbf{Guarantees for E-BalLOT.}
To what extent does our theory for BalLOT also work for E-BalLOT?
Unfortunately, several difficulties preclude an easy transfer.
In the unrealistic setting in which one could \textit{exactly} solve the entropically regularized Kantorovich problem in each E-BalLOT iteration, then 
E-BalLOT terminates in finitely many steps by an argument similar to the proof of Theorem~\ref{thm.finite termination}.
But the practical E-BalLOT implementation is different: Sinkhorn iterations are stopped at finite tolerance, and the resulting approximate coupling is rounded back to~$\bUnk$~\cite{altschuler2017near}, which creates extra complications for a rigorous termination analysis.
Separately, the planted-recovery arguments for exact BalLOT rely on cyclical monotonicity of an unregularized optimal coupling.
The entropic optimizer is fractional and satisfies a different optimality structure, known as cyclical invariance \cite{Nutz2021EntropicOT}.
The point is that E-BalLOT introduces complications that require more attention for a satisfying theoretical treatment.


\section*{Declarations}

The authors thank the referees for carefully reading our initial manuscript and providing several suggestions that greatly improved the presentation of our results.
DGM thanks OpenAI for complimentary access to GPT-5.6~Sol, which was used to help organize and evaluate an extensive revision of this paper.
DGM was supported by NSF DMS 2220304. 
There are no conflicts of interest.

\appendix

\section{Extreme points and tightness}
\label{sec:appendix}

\begin{fact}
\label{fact:extreme points of Unk}
When $k$ divides $n$, the extreme points of $\bUnk$ are precisely the integral couplings.
\end{fact}

\begin{proof}
First suppose $\bmF$ is integral and $\bmF=\frac{1}{2}(\bm{A}+\bm{B})$ for some 
$\bm{A},\bm{B}\in\bUnk$.
If $F_{ij}=0$, then nonnegativity forces $A_{ij}=B_{ij}=0$.
If $F_{ij}=1/n$, then the row-sum constraints and the already forced zero entries imply $A_{ij}=B_{ij}=1/n$.
Thus, $\bm{A}=\bm{B}=\bmF$, and every integral coupling is an extreme point.

Conversely, suppose $\bmF\in\bUnk$ is not integral.
Consider the bipartite graph whose left vertices are the row indices, whose right vertices are the column indices, and whose edges are the pairs $(i,j)$ for which $0<F_{ij}<1/n$.
Every vertex incident to one such edge is incident to at least two: a row has total mass $1/n$, while a column has total mass $1/k=(n/k)/n$, so neither an integer row sum nor an integer multiple of $1/n$ can be obtained from exactly one fractional entry together with entries in $\{0,1/n\}$.
A finite bipartite graph with minimum degree at least two on its nonisolated vertices contains an even cycle.

Let $C$ be such a cycle and define $\bm{E}$ by alternating $+\eta$ and $-\eta$ along the edges of $C$, with all other entries zero.
Choose $\eta>0$ small enough that $\bmF\pm\bm{E}$ remain nonnegative and no entry exceeds $1/n$.
Every row and column touched by the cycle receives one positive and one negative perturbation, and therefore $\bmF\pm\bm{E}\in\bUnk$.
Since $\bm{E}\neq0$,
\[
\bmF
=\frac{1}{2}(\bmF+\bm{E})+\frac{1}{2}(\bmF-\bm{E})
\]
is a nontrivial convex decomposition.
Thus, $\bmF$ is not an extreme point.
\end{proof}

\begin{proposition}
\label{prop:tightness of Kantorovich relaxation}
The relaxation from \eqref{problem: original_balanced} to \eqref{prob:main_obj} is tight.
\end{proposition}

\begin{proof}
By Fact~\ref{fact:extreme points of Unk}, for every fixed $\bmu$,
\[
\min_{\bmF\in\bUnk}f(\bmF,\bmu)
=
\min_{\substack{\bmF\in\bUnk\\\bmF\text{ integral}}}f(\bmF,\bmu),
\]
since a linear functional over a compact polytope attains its minimum at a vertex.
Minimizing both sides over $\bmu\in\Rdk$ proves the claim.
\end{proof}

\section{Diameter initialization can fail when $\Delta>2$}
\label{sec: diameter counterexample}

We give a deterministic example showing that the hypothesis $\Delta>2$ does not suffice for diameter initialization.
At separation $\Delta=2$, take empirical cluster centroids
\[
\hmu_1=(-1,0),
\qquad
\hmu_2=(1,0),
\]
and the two five-point clusters
\begin{align*}
C_1^\natural:
&\quad
(-1,-0.95),\ (-0.4,0.45),\ (-1,0.25),\ (-1.35,0),\ (-1.25,0.25),\\
C_2^\natural:
&\quad
(1,0.95),\ (0.4,-0.45),\ (1,-0.25),\ (1.35,0),\ (1.25,-0.25).
\end{align*}



\begin{figure}[htbp]
\centering
\begin{tikzpicture}[scale=1.2, every node/.style={font=\small}, transform shape]
    \draw[dashed] ( 2,0) circle (2);
    \draw[dashed] (-2,0) circle (2);
    \filldraw[red] ( 2,0) circle (2pt) node[above] {$\hmu_2 = (1,0)$};
    \filldraw[red] (-2,0) circle (2pt) node[below] {$\hmu_1 = (-1,0)$};
    \draw[->, red, thick] (-2,0) -- (2,0);
    \draw[->, blue, thick] (-2,-2+0.1) -- (2,2-0.1);
    \draw[->, green!60!black, thick] (-2,0) -- (-0.8,0.9);
    \draw[->, green!60!black, thick] (2,0) -- (0.8,-0.9);
    \filldraw (-2, -2+0.1) circle (2pt) node[above left, xshift = 10pt, font=\scriptsize] {$\bmu_1^0 = (-1,-0.95)$};
    \filldraw (-0.8, 0.9) circle (2pt) node[above, font=\scriptsize] {$\bm{x}_1 = (-0.4,0.45)$};
    \filldraw (-2, 0.5) circle (2pt) node[above, font=\scriptsize] {$(-1,0.25)$};
    \filldraw (-2.7, 0) circle (2pt) node[left, font=\scriptsize] {$(-1.35,0)$};
    \filldraw (-2.5, 0.5) circle (2pt) node[left, font=\scriptsize] {$(-1.25,0.25)$};
    \filldraw (2, 1.9) circle (2pt) node[below right, font=\scriptsize] {$(1,0.95) = \bmu_2^0$};
    \filldraw (0.8, -0.9) circle (2pt) node[below, font=\scriptsize] {$\bm{x}_2 = (0.4,-0.45)$};
    \filldraw (2, -0.5) circle (2pt) node[below, xshift = -5pt, font=\scriptsize] {$(1,-0.25)$};
    \filldraw (2.7, 0) circle (2pt) node[right, font=\scriptsize] {$(1.35,0)$};
    \filldraw (2.5, -0.5) circle (2pt) node[right, xshift=0pt, font=\scriptsize] {$(1.25,-0.25)$};
\end{tikzpicture}
\caption{Illustration of the balanced clustering geometry.}
\label{fig:balanced-clustering}
\end{figure}
Each cluster is contained in the unit ball centered at its displayed empirical centroid.
The unique diameter pair is
\[
\bmu_1^0=(-1,-0.95),
\qquad
\bmu_2^0=(1,0.95).
\]
Consider the points
\[
\bm{x}_1=(-0.4,0.45)\in C_1^\natural,
\qquad
\bm{x}_2=(0.4,-0.45)\in C_2^\natural.
\]
The cost of swapping these two points relative to the planted assignment is
\begin{align*}
\|\bm{x}_1-\bmu_2^0\|^2+\|\bm{x}_2-\bmu_1^0\|^2
&=4.42,
\end{align*}
whereas their planted assignment cost is
\begin{align*}
\|\bm{x}_1-\bmu_1^0\|^2+\|\bm{x}_2-\bmu_2^0\|^2
&=4.64.
\end{align*}
Thus, the first BalLOT assignment from the diameter initialization cannot recover the planted clustering.

The failure persists for separations strictly larger than $2$.
Translate every point of the right cluster by $(\eta,0)$ for some sufficiently small $\eta>0$.
The two empirical cluster centroids then have separation $2+\eta>2$.
Because the diameter pair is unique and the assignment cost gap $4.64-4.42=0.22$ is strict, continuity guarantees that the same diameter pair remains selected and the same swap remains favorable for all sufficiently small $\eta>0$.
Hence, $\Delta>2$ is insufficient.

\begingroup
\raggedright
\bibliography{references}
\endgroup

\end{document}